\documentclass[10pt]{article} 
\usepackage[preprint]{tmlr}

\usepackage{amsmath,amsfonts,bm}

\def\eqref#1{equation~\ref{#1}}

\def\1{\bm{1}}

\DeclareMathAlphabet{\mathsfit}{\encodingdefault}{\sfdefault}{m}{sl}
\SetMathAlphabet{\mathsfit}{bold}{\encodingdefault}{\sfdefault}{bx}{n}

\newcommand{\R}{\mathbb{R}}

\DeclareMathOperator*{\argmin}{arg\,min}

\DeclareMathOperator{\Tr}{Tr}

\usepackage{hyperref}
\usepackage{url}
\usepackage{hyperref}
\usepackage{siunitx}
\usepackage{amssymb}
\usepackage{subcaption}
\usepackage{graphicx}
\usepackage{algorithm,algpseudocode,multirow,hhline}
\usepackage{booktabs,multirow}
\usepackage{url}
\usepackage{amsthm}

\newtheorem{theorem}{Theorem}
\newtheorem{lemma}{Lemma}
\newtheorem{remark}{Remark}

\title{Tackling Feature and Sample Heterogeneity in Decentralized Multi-Task Learning: A Sheaf-Theoretic Approach}

\author{\name Chaouki Ben Issaid \email chaouki.benissaid@oulu.fi \\
      \addr Centre for Wireless Communications\\
      University of Oulu, Finland
      \AND
      \name Praneeth Vepakomma \email vepakom@mit.edu\\
      \addr MBZUAI and Massachusetts Institute of Technology
      \AND
      \name Mehdi Bennis \email mehdi.bennis@oulu.fi\\
      \addr Centre for Wireless Communications\\
      University of Oulu, Finland}

\newcommand{{\ours}}{\texttt{Sheaf-FMTL}}
\newcommand{\F}{\mathcal{F}}
\newcommand{\G}{\mathcal{G}}

\def\month{MM}  
\def\year{YYYY} 
\def\openreview{\url{https://openreview.net/forum?id=XXXX}} 

\begin{document}

\maketitle

\begin{abstract}
Federated multi-task learning (FMTL) aims to simultaneously learn multiple related tasks across clients without sharing sensitive raw data. However, in the decentralized setting, existing FMTL frameworks are limited in their ability to capture complex task relationships and handle feature and sample heterogeneity across clients. To address these challenges, we introduce a novel sheaf-theoretic-based approach for FMTL. By representing client relationships using cellular sheaves, our framework can flexibly model interactions between heterogeneous client models. We formulate the sheaf-based FMTL optimization problem using sheaf Laplacian regularization and propose the {\ours} algorithm to solve it. We show that the proposed framework provides a unified view encompassing many existing federated learning (FL) and FMTL approaches. Furthermore, we prove that our proposed algorithm, {\ours}, achieves a sublinear convergence rate in line with state-of-the-art decentralized FMTL algorithms. Extensive experiments show that although {\ours} introduces computational and storage overhead due to the management of interaction maps, it achieves substantial communication savings in terms of transmitted bits when compared to decentralized FMTL baselines. This trade-off makes {\ours} especially suitable for cross-silo FL scenarios, where managing model heterogeneity and ensuring communication efficiency are essential, and where clients have adequate computational resources.
\end{abstract}

\section{Introduction} \label{intro}
The growing demand for privacy-preserving distributed learning algorithms has steered the research community towards federated learning (FL) \citep{mcmahan2017communication}, a learning paradigm that allows several clients, such as mobile devices or organizations, to cooperatively train a model without revealing their raw data. By aggregating locally computed updates rather than raw data, FL aims to learn a global model that benefits from the different data distributions inherently present across the participating clients. Despite its promise, conventional FL faces significant hurdles when dealing with client data heterogeneity. In fact, while the global model may perform well on average, the statistically heterogeneous clients' data have been shown to affect the model's existence and convergence \citep{sattler2020clustered, Li2020}. These challenges are exacerbated in a decentralized environment where coordination is limited and direct control over the client models is not feasible. Recently, there have been several attempts to bring personalization into FL to learn distinct local models \citep{wang2019federated, fallah2020personalized, hanzely2020federated} since learning a personalized model per client is more suitable than a single global model to tackle data heterogeneity. These models are specifically learned to fit the heterogeneous local data distribution via techniques such as federated multi-task learning (FMTL) \citep{smith2017federated, dinh2022new} that model the interactions between the different personalized local models.

FMTL generalizes the FL framework by allowing the simultaneous learning of multiple related but distinct tasks across several clients. Unlike traditional FL, which focuses on training a single global model, FMTL acknowledges that different clients may be interested in solving distinct tasks that are related but not identical. By leveraging task relatedness, FMTL aims to improve the performance of individual task models through shared knowledge while maintaining task-specific uniqueness. This approach not only enhances the generalization performance of the models on individual tasks by leveraging shared information but also contributes to tackling the non-independent and identically distributed (non-IID) nature of data across clients. FMTL considers the different objectives and data distributions across clients, customizing models to perform optimally on each task while still benefiting from the federated structure of the problem as well as the similarity between these tasks.  However, existing FMTL frameworks are not without limitations. A major concern is the oversimplified view of task interdependencies, where relationships between tasks are often modeled using simple fixed scalar weights. This approach captures only the basic notion of task relatedness but fails to represent more intricate and higher-order dependencies that may exist among tasks. As a result, these models may overlook subtle interconnections and dynamic patterns of interdependence, leading to suboptimal knowledge sharing and reduced performance in heterogeneous and decentralized environments. For a comprehensive understanding of scenarios where task similarities are naturally defined in vector spaces, kindly refer to Appendix \ref{appendix:vector_space_scenarios}.  Furthermore, a critical issue with the current FMTL frameworks is the assumption that the models have the same size, which limits the applicability of FMTL in the case of different model sizes. These challenges are particularly pronounced in cross-silo FL scenarios, where collaborating organizations may possess large, complex datasets with significant heterogeneity and may use different model architectures, necessitating a flexible and robust learning framework. Last but not least, to the best of our knowledge, apart from MOCHA \citep{smith2017federated}, which requires the presence of a server, the interactions between the tasks are assumed to be known and not learned during training. Therefore, to address these challenges, we seek to answer the following question:

\textbf{``How can we effectively model and learn the complex interactions between various tasks/models in an FMTL decentralized setting, even in the presence of different model sizes?"}

To answer this question, the concept of sheaves provides a novel lens through which the interactions between clients in a decentralized FMTL setting can be modeled. The mathematical notion of a sheaf initially invented and developed in algebraic topology, is a framework that systematically organizes local observations in a way that allows one to make conclusions about the global consistency of such observations \citep{robinson2014topological, robinson2013understanding, riess2022diffusion}. As such requirements are a central part of FL problems, it is natural to ask how one could utilize a sheaf-based framework to find an effective solution to the above question. Given an underlying topology of the client relationships, we employ the notion of a cellular sheaf that captures the underlying geometry, enabling a richer and more nuanced multi-task learning environment. Sheaves enable the representation of local models as \textit{sections} over the underlying space, offering a structured way to capture the relationships between tasks/models in FMTL settings. As an inherent feature of the sheaf-based framework, our approach can support heterogeneity over local models. More specifically, our framework naturally facilitates learning models with different model dimensions. Moreover, sheaves are inherently distributed in nature and hence facilitate decentralized training. A sheaf data structure consists of vector spaces and linear mappings between them. In this work, we model the underlying space as a graph, and vector spaces are defined over vertices and edges, capturing pairwise interactions. Crucially, we are required to learn these maps that constitute the sheaf structure, as a decision variable of our problem. Learning these maps is instrumental in comparing heterogeneous models by projecting them onto a common space. 

\textbf{Contributions.} This paper introduces a novel unified approach that fundamentally rethinks FMTL and gives it a new interpretation by incorporating principles from sheaf theory. Our work primarily targets cross-silo FL environments. Unlike cross-device FL involving millions of resource-constrained edge devices, cross-silo settings typically feature fewer but more computationally capable participants with substantial local computing resources. This context is particularly well-suited for exploring richer collaborative learning frameworks that can effectively handle heterogeneity in exchange for some additional computational overhead. In what follows, we summarize our main contributions
\begin{itemize}
    \item Our proposed framework demonstrates a high degree of flexibility as it addresses the challenges arising from both feature and sample heterogeneity in the context of FMTL exploiting sheaf theory. It may be regarded as a comprehensive and unified framework for FMTL, as it encompasses a multitude of existing frameworks, including personalized FL \citep{hanzely2020federated}, conventional FMTL \citep{dinh2022new}, hybrid FL \citep{zhang2022hybrid}, and conventional FL \citep{mcmahan2017communication}.
    \item To the best of our knowledge, this is the first work that proposes to solve the FMTL in a decentralized setting while modelling higher-order relationships among heterogeneous clients. Furthermore, unlike existing decentralized FMTL frameworks, we learn the interactions between the clients as part of our optimization framework.
    \item Our algorithm, coined {\ours}, exhibits high communication efficiency, as the size of shared vectors among clients is significantly smaller in practice compared to the original models. This advantage is accompanied by computational and storage overhead due to the management of restriction maps, making the approach especially well-suited for cross-silo FL environments where clients have sufficient computational resources.
    \item A detailed convergence analysis of our proposed algorithm shows that the average squared norm of the objective function gradient decreases at a rate of $\mathcal{O}(1/K)$, where $K$ is the number of iterations, recovering the convergence rate of state-of-the-art FMTL \citep{smith2017federated,dinh2022new}.
    \item Extensive simulation results demonstrate the performance of our proposed algorithms on several benchmark datasets compared to state-of-the-art approaches.
\end{itemize}
\section{Related Work}\label{relatedwork}
\textbf{Federated Learning (FL).} FL is designed to train models on decentralized user data without sharing raw data. While numerous FL algorithms \citep{mcmahan2017communication, karimireddy2020scaffold, li2020federated, Lin2020, elgabli22a} have been proposed, most of them typically follow a similar iterative procedure where a server sends a global model to clients for updates. Then, each client trains a local model using its data and sends it back to the server for aggregation to update the global model. However, due to significant variability in locally collected data across clients, data heterogeneity poses a serious challenge. A prevalent assumption within FL literature is that the model size is the same across clients. However, recent works \citep{zhang2022hybrid, liu2022no} have highlighted the significance of incorporating heterogeneous model sizes in FL frameworks in the presence of a parameter server.

\textbf{Personalized FL (PFL).} To address the challenges arising from data heterogeneity, PFL aims to learn individual client models through collaborative training, using different techniques such as local fine-tuning \citep{wang2019federated, yu2020salvaging}, meta-learning \citep{fallah2020personalized,chen2018federated, jiang2019improving}, layer personalization \citep{arivazhagan2019federated, liang2020think, collins2021exploiting}, model mixing \citep{hanzely2020federated, deng2020adaptive}, and model-parameter regularization \citep{t2020personalized, li2021ditto, huang2021personalized, liu2022privacy}. One way to personalize FL is to learn a global model and then fine-tune its parameters at each client using a few stochastic gradient descent steps, as in \citep{yu2020salvaging}. Per-FedAvg \citep{fallah2020personalized} combines meta-learning with FedAvg to produce a better initial model for each client. Algorithms such as FedPer \citep{arivazhagan2019federated}, LG-FedAvg \citep{liang2020think}, and FedRep \citep{collins2021exploiting} involve layer-based personalization, where clients share certain layers while training personalized layers locally.  A model mixing framework for PFL, where clients learn a mixture of the global model and local models was proposed in \citep{hanzely2020federated}. In \citep{t2020personalized}, pFedMe uses an $L_2$ regularization to restrict the difference between the local and global parameters.

\textbf{Federated Multi-Task Learning (FMTL).} FMTL aims to train separate but related models simultaneously across multiple clients, each potentially focusing on different but related tasks. It can be viewed as a form of PFL by considering the process of learning one local model as a single task.  Multi-task learning was first introduced into FL in \citep{smith2017federated}. The authors proposed MOCHA, an FMTL algorithm that jointly learns the local models as well as a task relationship matrix, which captures the relations between tasks.  In the context of FMTL, task similarity can be represented through graphs, matrices, or clustering. In \citep{sattler2020clustered}, clustered FL, an FL framework that groups participating clients based on their local data distribution was proposed. The proposed method tackles the issue of heterogeneity in the local datasets by clustering the clients with similar data distributions and training a personalized model for each cluster. FedU, an FMTL algorithm that encourages model parameter proximity for similar tasks via Laplacian regularization, was introduced in \citep{dinh2022new}. In \citep{sarcheshmehpour2023networked}, the authors leverage a generalized total variation minimization approach to cluster the local datasets and train the local models for decentralized collections of datasets with an emphasis on clustered FL. 

\textbf{Data Heterogeneity in FL.} A fundamental challenge in FL is the inherent heterogeneity of data across clients, commonly referred to as non-IID data \citep{kairouz2019advances, li2020federated}. Following taxonomies proposed in the literature \citep{zhao2018federated, hsieh2020non, li2022federated}, data heterogeneity can be categorized into several distinct types. \textit{Feature distribution skew} (covariate shift) occurs when clients have different marginal feature distributions $P(X)$ while conditional label distributions $P(Y|X)$ remain consistent across clients \citep{quionero2009dataset, koh2021wilds}. This is common in scenarios where data collection environments differ, such as medical imaging equipment varying across hospitals \citep{antunes2022federated}. \textit{Label distribution skew} arises when label distributions $P(Y)$ vary across clients but feature distributions conditioned on labels $P(X|Y)$ remain similar \citep{wang2020federated, li2021model}. For instance, some clients may have more instances of certain classes than others, reflecting demographic or geographical differences \citep{hsu2019measuring}. \textit{Concept shift} occurs when the relationship between features and labels $P(Y|X)$ differs across clients due to local environmental factors or preferences \citep{karimireddy2020scaffold, luo2021no}. This manifests as either different feature manifestations for the same labels \citep{tan2022fedproto} or entirely different decision boundaries across clients for conceptually similar tasks \citep{collins2021exploiting}. \textit{Quantity skew} refers to a significant imbalance in the amount of data possessed by each client \citep{wang2021field, duan2019astraea}, creating a disparity in their contributions to the global model.

\textbf{Sheaves.}
A major limitation in FMTL over a graph, e.g., FMTL with graph Laplacian regularization in \citep{dinh2022new} and FMTL with generalized total variance minimization in \citep{sarcheshmehpour2023networked}, that we wish to address in this work, is their inability to deal with feature heterogeneity between clients. In contrast, \textit{sheaves}, a well-established notion in algebraic topology, can inherently model higher-order relationships among heterogeneous clients. Despite the limited appearance of sheaves in the engineering domain, their importance in organizing information/data distributed over multiple clients/systems has been emphasized in the recent literature \citep{robinson2014topological, robinson2013understanding, riess2022diffusion}. In fact, sheaves can be considered as the canonical data structure to systematically organize local information so that useful global information can be extracted \citep{robinson2017sheaves}. The above-mentioned graph models with node features lying in some fixed space can be considered as the simplest examples of sheaves, where such a graph is equivalent to a \textit{constant sheaf} structure that directly follows from the graph. Motivated by these ideas, our work focuses on using the generality of sheaves to propose a generic framework for FMTL with both data and feature heterogeneity over nodes, generalizing the works of \citep{dinh2022new, sarcheshmehpour2023networked}. The analogous generalization of the graph Laplacian in the sheaf context is the \textit{sheaf Laplacian}. In the context of distributed optimization, \citep{hansen2019distributed} consider sheaf Laplacian regularization with sheaf constraints, i.e., \textit{Homological Constraints}, and the resulting saddle-point dynamics. 
\section{Sheaf-based Federated Multi-Task Learning ({\ours})}\label{setup}
\subsection{FMTL Problem Setting}\label{setting}
We consider a connected network of $N$ clients modeled by a graph $\mathcal{G} = (\mathcal{V}, \mathcal{E})$, where $\mathcal{V} = [N] = \{1, \dots, N\}$ is the set of clients, and $\mathcal{E} \subseteq \mathcal{V} \times \mathcal{V}$ represents the set of edges, i.e., the set of pairs of clients that can communicate with each other. Each client $i \in \mathcal{V}$ has a local loss function $f_i: \mathbb{R}^{d_i} \rightarrow \mathbb{R}$ and only has access to its own local data distribution $\mathcal{D}_i$. Client $i$ can only communicate with the set of its neighbors defined as $\mathcal{N}_i = \{j \in \mathcal{V}|(i, j) \in \mathcal{E} \}$ whose cardinality is $|\mathcal{N}_i| = \delta_i$. In this work, we aim to fit different models, i.e., $\bm{\theta}_i \in \mathbb{R}^{d_i},$  $\forall i \in [N]$, to the local data of clients, while accounting for the interactions between these models. Finally, let $\bm{\theta} = [\bm{\theta}_1^T, \dots, \bm{\theta}_N^T]^T \in \mathbb{R}^{d}$ be the stack of the local decision variables $\{\bm{\theta}_i\}_{i=1}^N$, where $d = \sum_{i=1}^N d_i$.

\begin{figure*}[t]
    \centering
    \includegraphics[scale=0.35]{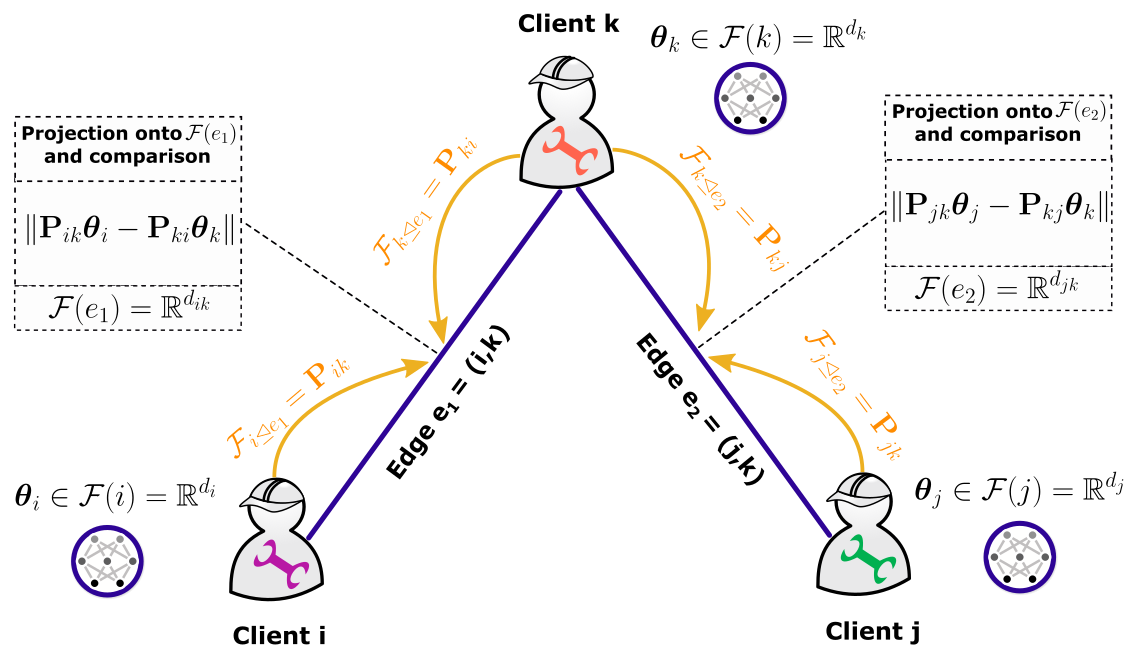}
    \caption{Schematic illustration of the sheaf-based modeling of FMTL.}
    \label{fig:sheaf_illustration}
\end{figure*}

\subsection{A Sheaf Theoretic Approach of the FMTL Problem}\label{sheaves}
A ``cellular sheaf" $\F$ of $\R$-vector spaces over a simple graph $\G = (\mathcal{V}, \mathcal{E})$, i.e., without loops and multiple edges, consists of the following assignments
\begin{itemize}
    \item For each $i \in \mathcal{V}$, a vector space $\F(i)=\R^{d_i}$ of dimension $d_i$,
    \item for each edge $e=(i,j) \in \mathcal{E}$, a vector space $\F(e)=\R^{d_{ij}}$ of dimension $d_{ij}$, and 
    \item for each edge $e \in \mathcal{E}$ and a vertex $i \in \mathcal{V}$ that is incident to the edge $e$, a linear transformation $\F_{i\trianglelefteq e} : \F(i)\rightarrow \F(e)$.
\end{itemize}
We shall refer to $\F(i)$ and $\F(e)$ as stalks over $i$ and $e$, respectively, and the map $\F_{i\trianglelefteq e}$ as the restriction map from $i$ to $e$. Also, given an edge $e=(i,j)\in \mathcal{E}$, we denote the matrix representation of $\F_{i\trianglelefteq e}$, with respect to a chosen basis such as the standard basis, by $\bm{P}_{ij}$. With an abuse of notation, we use  $\F_{i\trianglelefteq e}$ and $\bm{P}_{ij}$ interchangeably, as required, in the remainder of the paper.

Naturally associated with such a sheaf structure are the dual maps, $\F_{i\trianglelefteq e}^*: \F(e)\rightarrow \F(i)$, of the restriction maps. Note that we are using the identification of the dual of a finite-dimensional vector space with itself. It is a standard fact that the matrix representation of the dual map $\F_{i\trianglelefteq e}^*$ is given by the transpose of $\bm{P}_{ij}$. Similarly, with an abuse of notation, we shall also use $\F_{i\trianglelefteq e}^*$ and $\bm{P}_{ij}^T$ interchangeably. 

For each $i \in \mathcal{V}$, $\F(i)$ is the space in which the local model of client $i$ is parameterized, i.e., $\bm{\theta}_i \in \F(i)$. A choice $\{\bm{\theta}_i\}_{i \in \mathcal{V}}$ of local models lies in the total space $C^0(\F):=\bigoplus_{i \in \mathcal{V}} \F(i)$. Note that, we do not assume $d_i$ and $d_j$ to be the same for $i\neq j$. In particular, different clients can have different model sizes that could arise from feature heterogeneity and/or different learning tasks. Also note that an element of $C^0(\F)$ is not fully observable by a single client, as assumed in the FL setting.  

Therefore, any two models can be compared via the restriction maps, provided they share an edge. More specifically, as can be seen from Figure \ref{fig:sheaf_illustration}, for two clients $i$ and $j$ such that $e=(i,j) \in \mathcal{E}$, $\F(e)$ can be considered as the ``disclose space" in which models $\bm{\theta}_i$ and $\bm{\theta}_j$ are compared via the projections $\F_{i\trianglelefteq e} \left( \bm{\theta}_i \right) = \bm{P}_{ij}$ and $\F_{j\trianglelefteq e} \left( \bm{\theta}_j \right) = \bm{P}_{ji}$. Here, the local models \( \bm{\theta}_i \) are assigned to the vertices (the clients), while the restriction maps \( \bm{P}_{ij} \) project the local models onto the edge space capturing the shared features or relationships between the clients.

We refer the reader to Appendix \ref{appendix:discussion} for an interpretation of this viewpoint in the context of linear models. Given a choice of local models, the overall comparison of these models is done in the total space $C^1(\F):= \bigoplus_{e\in \mathcal{E}} \F(e)$ of the disclose spaces. The total discrepancy of such a choice of local models, as measured by comparing their projections onto the disclose spaces, can be formalized via the Laplacian quadratic form associated with the so-called ``sheaf Laplacian", the analogous to the graph Laplacian. To define the Laplacian in the sheaf setting, we first need to orient the edges and define the ``co-boundary map" $\delta: C^0(\F) \rightarrow C^1(\F)$. 

From now onwards, we shall fix an orientation for each edge and write $e=(i,j)$ for an oriented edge. For such an oriented edge $e=(i,j)$, write $e^+=j$ and $e^-=i$. Our discussion is not subjective to the choice of orientation; however, one can choose a canonical orientation associated with an ordering of the vertices, e.g., when vertices are indexed by numbers, by choosing $e^+=\max\{i,j\}$ and $e^-=\min\{i,j\}$ for an unoriented edge $e=\{i,j\}$. Given such an orientation, the co-boundary map $\delta$ is given as follows. For $\bm{\theta} = (\bm{\theta}_i)_{i \in \mathcal{V}}$, the co-boundary of $\bm{\theta}$,   $\delta\left(\bm{\theta} \right) = \left( \delta\left(\bm{\theta} \right) _e \right)_{e \in \mathcal{E}} \in C^1(\F)$, is defined by 
\begin{align}
\delta\left(\bm{\theta} \right) _e = \F_{e^+\trianglelefteq e}\left( \bm{\theta}_{e^+} \right) - \F_{e^-\trianglelefteq e}\left( \bm{\theta}_{e^-} \right).    
\end{align}
As in the case of restriction maps, one also has the dual $\delta^*: C^1(\F) \rightarrow C^0(\F)$ of the co-boundary map $\delta$. The sheaf Laplacian $L_\F : C^0(\F) \rightarrow C^0(\F)$ associated with the cellular sheaf $\F$ is then given by $L_\F = \delta^* \circ \delta$. For a given $\bm{\theta}$, the sheaf Laplacian is defined as $L_\F\left(\bm{\theta}\right) = \left( L_\F\left(\bm{\theta}\right)_j\right)_{j \in \mathcal{V}}$, where
$L_\F\left(\bm{\theta}\right)_j = \sum_{i \in \mathcal{V}} L_{j,i}\left(\bm{\theta}_i\right)$ is given by 
\begin{align}
\label{eqn: Laplacian_defn}
    L_{j,i} = \left\{
    \begin{array}{ll}
          \sum\limits_{i\trianglelefteq e} \F_{i\trianglelefteq e}^* \circ \F_{i\trianglelefteq e}, & \text{if } i = j, \\
           - \F_{j\trianglelefteq e}^* \circ \F_{i\trianglelefteq e} , & \text{if } e=(i, j) \in \mathcal{E},\\
           \bm{0}, & \text{otherwise.}
         \end{array}
    \right.
\end{align}
In particular, based on the ordering of $\mathcal{V}$ that is used to stack the local models, and the chosen bases for $\F(i)$'s and $\F(e)$'s , $L_\F$ is a block matrix structure indexed by $\mathcal{V}$, whose $(j,i)$ block can be directly obtained from (\ref{eqn: Laplacian_defn}). More specifically, with an abuse of notation, writing $L_{j,i}$ in terms of its matrix representation, we have that
\begin{align}
\label{eqn: Laplacian_matrix_form}
    L_{j,i} = \left\{
    \begin{array}{ll}
          \sum\limits_{j' \in  \mathcal{N}_i } \bm{P}_{ij'}^T\bm{P}_{ij'}, & \text{if } i = j, 
          \\
           - \bm{P}_{ji}^T\bm{P}_{ij}, & \text{if } e=(i, j) \in \mathcal{E},\\ 
           \bm{0}, & \text{otherwise.}
         \end{array}
    \right.
\end{align}
In fact, one can write the matrix representation of the co-boundary map, $\bm{P}$, as follows. It has a block structure whose rows are indexed by edges and columns are indexed by vertices. Then, the submatrix $\bm{P}_{e,i}$ that corresponds to row $e\in \mathcal{E}$ and column $i \in \mathcal{V}$ is given by
\begin{align}
\label{eqn: Boundary_matrix_form}
    \bm{P}_{e,i} = \left\{
    \begin{array}{ll}
          \bm{P}_{ij}, & \text{if } e =(i,j) 
          \\
           - \bm{P}_{ij}, & \text{if } e=(j, i) ,\\ 
           \bm{0}, & \text{otherwise.}
         \end{array}
    \right.
\end{align}
From the definition $L_\F = \delta^* \circ \delta$, one can get the matrix form of $L_\F$ as $L_\F = \bm{P}^T \bm{P}$. Assuming the matrix representation of $L_\F$, we often write $L_\F\left(\bm{\theta}\right) = L_\F\,\bm{\theta} = \bm{P}^T \bm{P} \bm{\theta}$. Note that this block structure aligns well with the distributed optimization goal of the FMTL setting.

The significance of the sheaf Laplacian $L_\F$ is characterized by the consensus property given by
\begin{align}
    \ker(L_\F) = \left\{ \left(\bm{\theta}_i\right)_{i \in \mathcal{V}} \in C^0(\F) \,|\, \F_{i\trianglelefteq e}\left( \bm{\theta}_i \right)   =  \F_{j\trianglelefteq e}\left( \bm{\theta}_j \right)  \text{  for  } e=(i,j)\in \mathcal{E}\right\}.
\end{align}
In other words, $\ker(L_\F)$ consists of the choices of local models that are in \textit{global consensus} so that any two comparable local models $\bm{\theta}_i$ and $\bm{\theta}_j$ agree when projected onto the disclose space $\F(e)$, where $e=(i,j) \in \mathcal{E}$. Accordingly, the global consensus constraint on $\bm{\theta} \in  C^0(\F)$ is given by $L_\F\, \bm{\theta}  = \bm{0}$.

Similar to that of a graph Laplacian, the sheaf Laplacian quadratic form $Q_\F(\bm{\theta}) = \bm{\theta}^T L_\F\, \bm{\theta}$ quantifies by how much a given $\bm{\theta}$ deviates from the constraint $L_\F\, \bm{\theta}  = \bm{0}$. The following lemma shows that the sheaf Laplacian quadratic form measures the total discrepancy between the projections of the local models onto the edge spaces, summed over all edges in the graph.
\begin{lemma}\label{lemma1}
\begin{align}
    \bm{\theta}^T L_\F \, \bm{\theta} = \sum_{e=(i,j) \in \mathcal{E}} \left\lVert \F_{i\trianglelefteq e}\left( \bm{\theta}_i \right)  -  \F_{j\trianglelefteq e}\left( \bm{\theta}_j \right) \right\rVert^2 = \sum_{e=(i,j) \in \mathcal{E}} \left\| \bm{P}_{ij} \bm{\theta}_{i} - \bm{P}_{ji} \bm{\theta}_{j} \right\|^2.
\end{align}    
\end{lemma}
\begin{proof}
The details of the proof are deferred to Appendix \ref{suppLemmas}.
\end{proof}
Next, we show that $\bm{\theta}$ being a global section, i.e., $\bm{\theta} \in \ker(L_\F)$, is equivalent to $\bm{\theta}$ minimizing the sheaf Laplacian quadratic form.
\begin{lemma}
\begin{align}
\ker(L_\F) = \argmin_{\bm{\theta} \in  C^0(\F) }  Q_\F(\bm{\theta}) = \argmin_{\bm{\theta}\in  C^0(\F) } \bm{\theta}^T L_\F\, \bm{\theta}.    
\end{align}    
\end{lemma}
\begin{proof}
 The proof is provided in Appendix \ref{suppLemmas}.
\end{proof}
In the context of FMTL, a cellular sheaf is a structured way to assign information to each client (node) and their connections (edges) in a network as follows
\begin{itemize}
\item Nodes (Clients): Each client $i$ has its own local model $\bm{\theta}_i \in \mathbb{R}^{d_i}$.
\item Edges (Interaction space): Each connection between clients $i$ and $j$ has a shared space $\mathcal{F}(e) \in \mathbb{R}^{d_{ij}}$.
\item Restriction Maps: The mappings $\bm{P}_{ij}: \mathbb{R}^{d_i} \rightarrow \mathbb{R}^{d_{ij}}$ project the local model of client $i$ into the shared space with client $j$. This projection facilitates meaningful comparisons and collaborations between clients, ensuring that heterogeneous models can still interact effectively within the network.
\end{itemize} 
This setup allows us to systematically ensure that the models of connected clients align within their shared interaction spaces, promoting consistency and cooperation across the network. To formulate the FMTL optimization problem, we aim to minimize the combined objectives of individual client losses and a regularization term that enforces consistency across client models. Specifically, each client seeks to minimize its own loss function based on its local data, while the regularization term penalizes discrepancies between connected clients in their shared interaction spaces. Building upon this, we can express the regularization term more succinctly using the sheaf Laplacian matrix \( L_{\mathcal{F}}(\bm{P}) \). This matrix encapsulates the structural relationships and shared interaction spaces between clients, allowing us to reformulate the optimization problem in a compact and mathematically elegant manner as follows
\begin{align}\label{P2}
    \min_{\bm{\theta}, \bm{P} } \Psi(\bm{\theta}, \bm{P}) = f(\bm{\theta}) + \frac{\lambda}{2} \bm{\theta}^T L_{\mathcal{F}}(\bm{P}) \bm{\theta},
\end{align}
where $f(\bm{\theta}) = \sum_{i=1}^N f_i(\bm{\theta}_i)$ and $L_{\mathcal{F}}(\bm{P})$ is the sheaf Laplacian for the choices of the restriction maps. The sheaf Laplacian regularization term \( \frac{\lambda}{2} \bm{\theta}^T L_{\mathcal{F}}(\bm{P}) \bm{\theta} \) enforces consistency between the projections of the local models onto the edge space, promoting collaboration among the clients. A more in-depth discussion on the rationale behind using Sheaf theory to model clients' interaction in the context of FMTL can be found in Appendix \ref{app:sheaf_theory}. In the next subsection, we show that our proposed framework is very general and covers many special cases previously introduced in the literature.
\subsection{Special Cases}\label{appendix:specialcases}
In this section, we show how some previous works can be considered as special cases of our framework. Most of the works that we refer to, except \citep{zhang2022hybrid}, consider the same model size for all the clients. Thus, unless otherwise stated, we assume in the rest of this subsection that all the local models are assumed to have the same size $p \in \mathbb{N}$, i.e., $\forall i \in \mathcal{V}, \bm{\theta}_i \in \mathbb{R}^p$. In particular, we chose $\forall i \in \mathcal{V}$ and $ \forall e\in \mathcal{E}$, $\F(i) = \F(e) = \mathbb{R}^{p}$. 
\\
\textbf{Connection with conventional FMTL \citep{dinh2022new}.} 
In conventional FMTL, we aim to solve the problem
\begin{align}
    \min_{\{\bm{\theta}_i\}_{i \in \mathcal{V}}} \sum_{i=1}^N f_i(\bm{\theta}_i) + \frac{\lambda}{2} \sum_{i=1}^N \sum_{j \in \mathcal{N}_i} a_{ij} \left\|  \bm{\theta}_{i} - \bm{\theta}_{j} \right\|^2, \label{eqn: conventional FMTL objective}
\end{align}
where the weights $\{a_{ij}\}$ are assumed to be known in advance. This problem is a special case associated with the sheaf $\F$ arising by choosing $\F_{i \trianglelefteq e} = \bm{P}_{ij} = \sqrt{a_{ij}} \bm{I}_{p}$, where $\bm{I}_{p}$ is the $p \times p$ identity map/matrix. The dimension of the projection space is $d_{ij} = p$ for all edges $(i, j)$. For this particular choice of the sheaf $\F$, the associated Laplacian quadratic form is precisely 
    \begin{align}
     Q_{\F}\left( \bm{\theta} \right)= \bm{\theta}^T {L}_{\F} \, \bm{\theta} = \sum_{i,j \triangleleft e} a_{ij}\left\lVert  \bm{\theta}_i - \bm{\theta}_j \right \rVert^2 = \sum_{(i,j) \in \mathcal{E}} a_{ij} \left\| \bm{\theta}_i - \bm{\theta}_j \right\|^2.  
    \end{align} 
Replacing this into (\ref{P1}), we get the conventional FMTL problem (\ref{eqn: conventional FMTL objective}).\\
\textbf{Connection with conventional FL \citep{ye2020decentralized}.} 
Setting the restriction maps to be $\bm{P}_{ij} = \bm{I}_{p}$ and $d_{ij} = p, ~\forall (i,j) \in [N] \times [N]$. Then, the sheaf Laplacian regularization is given by
\begin{align}
    Q_\mathcal{F}(\bm{\theta}) = \sum_{(i,j) \in \mathcal{E}} \left\| \bm{\theta}_i - \bm{\theta}_j \right\|^2.
\end{align}
Hence, (\ref{P1}) reduces to
\begin{align}
    \min_{\{\bm{\theta}_i\}_{i \in \mathcal{V}}} \sum_{i=1}^N f_i(\bm{\theta}_i) + \frac{\lambda}{2} \sum_{(i,j) \in \mathcal{E}} \left\| \bm{\theta}_i - \bm{\theta}_j \right\|^2
\end{align}
Taking $\lambda \rightarrow \infty$, as pointed out in Remark \ref{remark1}, one recovers the conventional FL problem
\begin{align}
    \nonumber &\min_{\{\bm{\theta}_i\}_{i \in \mathcal{V}}} \sum_{i=1}^N f_i(\bm{\theta}_i) \\
    &~\text{s.t } \bm{\theta}_i = \bm{\theta}_j, \forall (i,j) \in \mathcal{E}.
\end{align}
\textbf{Connection with personalized FL \citep{hanzely2020lower}.} Introducing the client $0$, e.g., a server, where $f_0 \triangleq 0$ and $\bm{\theta}_0 = \bm{\bar{\theta}}$. Furthermore, let $\bm{P}_{0i} = \bm{P}_{i0} =  \bm{I}_p, ~\forall i \in [N]$, and  $\bm{P}_{ij} = \bm{0}_p, ~\forall (i,j) \in [N] \times [N]$, where $\bm{0}_p$ is the $p \times p$ zero map/matrix. Hence, the set of neighbours of client $0$ is $\mathcal{N}_0 = [N]$, and the set of each client $i \in [N]$ is $\mathcal{N}_i = \{0\}$. We observe that this amounts to choosing the constant sheaf over the graph that connects each client to the server. Then, the associated Laplacian quadratic form can be written as
\begin{align}
Q_{\F}\left( \bm{\theta} \right) =  \sum_{i \in \mathcal{N}_0} \left\| \bm{P}_{0i} \bm{\theta}_{0} - \bm{P}_{i0} \bm{\theta}_{i} \right\|^2 =  \sum_{i=1}^N \left\|\bm{\bar{\theta}} - \bm{\theta}_{i} \right\|^2.  
\end{align}
Therefore, (\ref{P1}) reduces to the personalized FL objective \citep[Eq. (2)]{hanzely2020lower}
\begin{align}
    \min_{\{\bm{\theta}_i\}_{i \in \mathcal{V}}} \sum_{i=1}^N f_i(\bm{\theta}_i) + \frac{\lambda}{2} \sum_{i=1}^N  \left\|  \bm{\theta}_{i} - \bar{\bm{\theta}} \right\|^2.
\end{align}\\
\textbf{Connection with hybrid FL \citep{zhang2022hybrid}.}
In this case,  a communication framework is established between a server (with index $0$) and a set of clients ($i \in [N]$), with each client connected solely to the server. The server has access to all features, while clients are constrained by their local features. Hence, for each client $i$, $d_i \leq d_0$, where $\bm{\theta}_i \in \mathbb{R}^{d_i}$ and $\bm{\theta}_0 \in \mathbb{R}^{d_0}$ are the models of client $i$ and the server, respectively. Let $\bm{\Pi}_i$ denote binary matrices that prune the server model to align with the client local model, referred to as the \emph{selection matrices}. Given the above description, we have $\F(i)=\mathbb{R}^{d_i}$ and $\F(e)=\mathbb{R}^{d_i}$ for every $i \in [N]$ and edge of the form $e=(i,0)$. The associated restriction maps are $ \bm{P}_{0i} = \bm{\Pi}_i $ and $\bm{P}_{i0}=\bm{I}_{d_i}$, for $i\in [N]$ and $\bm{P}_{ij} = \bm{0}_p, ~\forall (i,j) \in [N] \times [N]$. With these choices, the Laplacian quadratic form of this sheaf is equal to the regularizer term  
\begin{align}
\label{eqn: regularizer_2_of_zang}
     Q_{\F}\left(\bm{\theta}, \bm{\Pi} \right) = \sum_{i=1}^N \left\lVert \bm{\theta}_i - \bm{\Pi}_i \bm{\theta}_0 \right\rVert^2.
\end{align}
Replacing this into (\ref{P1}), we get the hybrid FL objective \citep[Eq. (6) given $\mu_1 = 0$]{zhang2022hybrid}.
\subsection{Proposed Algorithm \& Convergence Analysis} \label{algorithm}
Note that problem (\ref{P2}) arising from the sheaf formulation can be re-written as follows
\begin{align}\label{P1}
    \min_{ \substack{\{\bm{\theta}_i\}_{i \in \mathcal{V}}, \\ \{\bm{P}_{ij}\}_{(i,j) \in \mathcal{E}}}} \sum_{i=1}^N f_i(\bm{\theta}_i) + \frac{\lambda}{2} \sum_{i=1}^N \sum_{j \in \mathcal{N}_i} \left\| \bm{P}_{ij} \bm{\theta}_{i} - \bm{P}_{ji} \bm{\theta}_{j} \right\|^2,
\end{align} 
where $\|\cdot\|$ is the Euclidean norm, and $\forall (i,j) \in \mathcal{E}$, $\bm{P}_{ij}$ is a matrix with size $(d_{ij}, d_i)$ such that $d_{ij} = d_{ji}$. In (\ref{P1}), we propose to jointly learn the models $\{\bm{\theta}_i\}_{i \in \mathcal{V}}$ and the matrices $\{\bm{P}_{ij}\}_{(i,j) \in \mathcal{E}}$. The matrix $\bm{P}_{ij}$ can be seen as an encoding or compression matrix since it maps the higher-dimensional vector $\bm{\theta}_i$ to a lower-dimensional space with dimension $d_{ij}$, effectively retaining only the most important features or information shared between the two clients $i$ and $j$. Hence, the term $(\bm{P}_{ij} \bm{\theta}_i - \bm{P}_{ji} \bm{\theta}_j)$ captures the dissimilarity or discrepancy between the two vectors $\bm{\theta}_i$ and $\bm{\theta}_j$ in this shared subspace. 
\begin{remark} \label{remark1}
In (\ref{P1}), the hyperparameter $\lambda$ controls the impact of the models of neighboring clients on each local model. When $\lambda > 0$, the minimization of the regularization term promotes the proximity among the models of neighboring clients. On the other hand, if $\lambda = 0$, (\ref{P1}) reduces to an individual learning problem, wherein each client independently learns its local model $\bm{\theta}_i$ solely from its local data, without engaging in any collaborative efforts with the other clients. Finally, as $\lambda \rightarrow \infty$, (\ref{P1}) boils down to the classical FL problem where the aim is to learn a global model \citep{hanzely2020federated}. 
\end{remark}

Next, we propose a communication-efficient algorithm to solve (\ref{P1}) by adopting an iterative optimization approach. Since the objective function is assumed to be differentiable, we can use gradient-based optimization methods. More specifically, we will use alternating gradient descent (AGD) updates for $\{\bm{\theta}_i\}_{i \in [N]}$ and $\{\bm{P}_{ij}\}_{(i,j) \in \mathcal{E}}$, respectively. At iteration $(k+1)$, client $i$ sends $\bm{P}_{ij}^{k} \bm{\theta}_i^{k}$ and receives $\{\bm{P}_{ji}^{k} \bm{\theta}_j^{k}\}_{j \in \mathcal{N}_i}$ from its neighbours, to update its model $\bm{\theta}_i$, using one gradient descent step
\begin{align}\label{thetaupdate}
    \bm{\theta}_i^{k+1}\!=\!\bm{\theta}_i^{k}\!-\!\alpha \left(\nabla f_i(\bm{\theta}_i^k)\!+\!\lambda \sum_{j \in \mathcal{N}_i} (\bm{P}_{ij}^k)^T ( \bm{P}_{ij}^k \bm{\theta}_i^{k}\!-\!\bm{P}_{ji}^k \bm{\theta}_j^{k}) \right).
\end{align}
Then, client $i$ sends $\bm{P}_{ij}^{k} \bm{\theta}_i^{k+1}$ and receives $\{\bm{P}_{ji}^{k} \bm{\theta}_j^{k+1}\}_{j \in \mathcal{N}_i}$ from its neighbours, to be able to update its matrices $\{\bm{P}_{ij}\}_{j \in \mathcal{N}_i}$, using one gradient descent step, according to
\begin{align}\label{pijaupdate}
    \bm{P}_{ij}^{k+1}\!=\!\bm{P}_{ij}^{k}\!-\! \eta \lambda  (\bm{P}_{ij}^{k} \bm{\theta}_i^{k+1} \!-\!\bm{P}_{ji}^{k} \bm{\theta}_j^{k+1}) (\bm{\theta}_i^{k+1})^T,
\end{align}
where $\alpha$ and $\eta$ are two learning rates. Note that, in our proposed algorithm, neighbouring clients only share vectors and no matrix exchange is needed in both updates (\ref{thetaupdate}) and (\ref{pijaupdate}). Our proposed method is summarized in Algorithm \ref{algo}.
\begin{remark}
Note that each neighbour of the node $i$ is required to send the vector $\bm{P}_{ji} \bm{\theta}_j$ in order to update $\bm{\theta}_i$ and $\bm{P}_{ij}$ as per (\ref{thetaupdate}) and (\ref{pijaupdate}). The dimension of $\bm{P}_{ji} \bm{\theta}_j$ is $d_{ij}$, which in practice could be much smaller than $d_j$, the size of $\bm{\theta}_j$. For example, a reasonable choice of $d_{ij}$ is $d_{ij} = \min(d_i, d_j)$. Hence, our proposed algorithm is more communication-efficient than sending the models $\{\bm{\theta}_i\}_{i \in \mathcal{V}}$.
\end{remark}
\begin{remark}
It is crucial to initialize the restriction maps $\bm{P}_{ij}^0$ with non-zero values (e.g., randomly, as done in our experiments). If initialized as zero matrices, the gradient update for $\bm{P}_{ij}$
in (\ref{pijaupdate}) would always be zero, preventing the learning of interactions and causing the algorithm to reduce to independent local training.   
\end{remark}

\begin{algorithm}[tb]
   \caption{Sheaf-based Federated Multi-Task Learning ({\ours})}
   \label{algo}
   \begin{algorithmic}[1]
      \State \textbf{Parameters:} Number of clients $N$, number of iterations $K$, learning rates $(\alpha, \eta)$, regularization parameter $\lambda$.
      \State \textbf{Initialization:} Initial models $\{\bm{\theta}_i^0\}_{i=1}^N$, initial matrices $\{\bm{P}_{ij}^0\}_{(i,j) \in \mathcal{E}}$.
      \For{$k = 0, \dots, K$}
         \For{$i = 1, \dots, N$ \textbf{in parallel}}
            \State $\triangleright$ Sends $\bm{P}_{ij}^{k} \bm{\theta}_i^{k}$ and receives $\bm{P}_{ji}^{k} \bm{\theta}_j^{k}$ from neighbors $j \in \mathcal{N}_i$
            \State $\triangleright$ Updates its model:
            \[
              \bm{\theta}_i^{k+1} = \bm{\theta}_i^{k} - \alpha \left(\nabla f_i(\bm{\theta}_i^k) + \lambda \sum_{j \in \mathcal{N}_i} (\bm{P}_{ij}^k)^T ( \bm{P}_{ij}^k \bm{\theta}_i^{k} - \bm{P}_{ji}^k \bm{\theta}_j^{k}) \right)
            \]
            \State $\triangleright$ Sends $\bm{P}_{ij}^{k} \bm{\theta}_i^{k+1}$ and receives $\bm{P}_{ji}^{k} \bm{\theta}_j^{k+1}$ from neighbors $j \in \mathcal{N}_i$
            \State $\triangleright$ Updates its matrix:
            \[
              \bm{P}_{ij}^{k+1} = \bm{P}_{ij}^{k} - \eta \lambda (\bm{P}_{ij}^{k} \bm{\theta}_i^{k+1} - \bm{P}_{ji}^{k} \bm{\theta}_j^{k+1}) (\bm{\theta}_i^{k+1})^T
            \]
         \EndFor
      \EndFor
   \end{algorithmic}
\end{algorithm}
Next, we turn to analyzing the convergence of {\ours}. To this end, we start by making the following standard assumptions.

\textbf{Assumption 1 (Smoothness).} $\forall i \in [N]$, the function $f_i$ is assumed to be $L$-smooth, i.e., there exists $L > 0$ such that $\forall i \in [N]$, $\forall \bm{\theta}_1, \bm{\theta}_2$, $\|\nabla f_i(\bm{\theta}_2) - \nabla f_i(\bm{\theta}_1)\| \leq L \|\bm{\theta}_2-\bm{\theta}_1\|$.

\textbf{Assumption 2 (Bounded domain).} There exists $D_{\theta} > 0$ such that $\|\bm{\theta}\| \leq D_{\theta}$.

Assumptions 1-2 are key assumptions that are often used in the context of distributed optimization \citep{karimireddy2020scaffold, li2020federated, deng2020distributionally, deng2024distributed}. In particular, Assumption 2 is commonly used in the convex-concave minimax literature \citep{deng2020distributionally, deng2024distributed}. The following theorem establishes the convergence rate of the {\ours} algorithm.
\begin{theorem}\label{theorem}
Let Assumptions 1 and 2 hold, and the learning rates $\alpha$ and $\eta$ satisfy the conditions $\alpha < \frac{2}{N L}$ and $\eta < \frac{2}{\lambda D_\theta^2}$, respectively. Then, the averaged gradient norm is upper bounded as follows
\begin{align}
\frac{1}{K}\sum_{k=0}^{K-1} \|\nabla \Psi(\bm{\theta}^k, \bm{P}^k)\|^2 \leq \frac{1}{\rho K} (\Psi(\bm{\theta}^{0}, \bm{P}^0) - \Psi^\star),
\end{align}
where $\rho = \min\left\{\alpha \left(1 - \frac{\alpha N L}{2}\right), \eta  \left(1 - \frac{\eta \lambda D_{\theta}^2}{2} \right) \right\}$ and $\Psi^\star$ is the optimal value of $\Psi$.
\end{theorem}
\begin{proof}
The proof can be found in Appendix \ref{theorem_proof}.
\end{proof}
Theorem \ref{theorem} shows that the {\ours} algorithm converges to a stationary point of the objective function at a rate of $\mathcal{O}(1/K)$. The proof involves two main steps. First, we study the descent step in the variable $\bm{\theta}^{k+1}$ and then, we study the descent step in the variable $\bm{P}^{k+1}$.
\section{Experiments} \label{experiments} 
\subsection{Experimental Setup} \label{experimentalsetup}
To validate our theoretical foundations, we numerically evaluate the performance of our proposed algorithm {\ours} using two experiments: (i) the clients have the same model size in Section \ref{experiment1}, and (ii) the clients have different model sizes in Section \ref{experiment2}. \\
\textbf{Datasets.} In the first experiment, we examine two datasets: rotated MNIST and heterogeneous CIFAR-10. A detailed description of the datasets can be found in Appendix \ref{appendix:dataset}. In appendix \ref{appendix:moreResults}, we report additional results using four more datasets.\\ 
\textbf{Model architecture.} For image datasets, we employ CNN architectures: rotated MNIST uses a CNN with two convolutional layers (32 and 64 filters) followed by a fully connected layer outputting 10 classes, while heterogeneous CIFAR-10 uses a similar architecture adapted for RGB inputs. For tabular datasets, we use multinomial logistic regression for classification tasks and linear models for regression tasks. In Sections \ref{experiment1} and \ref{ablation}, where clients have the same model size, the entire network architecture, including the final classification layer, is identical across all clients. In Section \ref{experiment2}, we investigate FMTL in a heterogeneous setting where clients have different model architectures. To simulate devices with varying computational capabilities, we employ different CNN architectures summarized in Table \ref{tab:heterogeneous_cnn_architectures}. These models are distributed among clients in almost equal proportions, creating a heterogeneous federation. \\
\textbf{Evaluation.} In the first experiment, we use a train/test split ratio of $75\%/25\%$ for all datasets, as done in \citep{smith2017federated}, where the test set for each client maintains the same data distribution characteristics (e.g., feature skew, label distribution) as their respective training set. Similar to \citep{dinh2022new}, we reduce the data size by $80\%$ for half of the clients to mimic the real-world FL setting where some clients have small datasets and can benefit from collaboration. For the regression tasks, we consider the loss function to be the regularized $L_2$ loss, while we use the $L_2$-regularized cross-entropy loss for the classification tasks. For each experiment, we report both the average test accuracy/MSE and its corresponding one standard error shaded area based on five runs. \\
\textbf{Parameters and hyperparameters.} In Sections \ref{experiment1} and \ref{ablation}, where all clients have the same model dimension $d$, we set $d_{ij} = \lfloor \gamma d \rfloor$ for all edges $(i, j) \in \mathcal{E}$, where $\gamma \in (0, 1]$ controls the communication/computation trade-off. For experiments with heterogeneous model dimensions $d_i \neq d_j$ (Section~\ref{experiment2}), we ensure compatibility and symmetry ($d_{ij}=d_{ji}$) by setting $d_{ij} = \lfloor \gamma \min(d_i, d_j) \rfloor$. In all our experiments, unless it is otherwise stated, the graph topology is based on the Erdős-Rényi model, where we randomly generate a network consisting of $N$ clients with a connectivity ratio $p=0.15$ for rotated MNIST and heterogeneous CIFAR-10 datasets and $p=0.2$ for the rest of the datasets. Both learning rates are chosen to be small to satisfy the conditions in Theorem \ref{theorem}. Furthermore, {\ours} uses the same global learning rate ($\alpha$) as dFedU to update the models for a fair comparison. The linear maps are initialized randomly from a normal distribution with a mean of 0 and a variance of 1. The values of the regularization parameter ($\lambda$) used for every dataset are listed in Table \ref{table:regularization}. Finally, learning rates and other relevant hyperparameters were tuned for all baseline methods using a grid search on a validation set to ensure a fair comparison.
\begin{table}[t]
\centering
\caption{
  Regularization parameters used during training.
  }
\label{table:regularization} 
\vspace{0.5em}
\begin{tabular}{@{}lcccccc@{}}
\toprule
\textbf{Dataset} & \multicolumn{1}{l}{HAR} & \multicolumn{1}{l}{Vehicle} & \multicolumn{1}{l}{GLEAM} & \multicolumn{1}{l}{School} & \multicolumn{1}{l}{R-MNIST} & \multicolumn{1}{l}{H-CIFAR-10} \\ \midrule
Regularization parameter ($\lambda$) & $0.05$ & $0.001$ & $0.001$ & $0.01$ & $0.001$ & $0.001$  \\ \bottomrule
\end{tabular}
\end{table}

\begin{table}[t] 
    \centering
    \caption{Summary of heterogeneous CNN architectures employed in Section~\ref{experiment2} for rotated MNIST and heterogeneous CIFAR-10 datasets, simulating clients with varying computational capabilities. The total number of clients is $N$.}
    \label{tab:heterogeneous_cnn_architectures}
    \begin{tabular}{@{}llccc@{}}
        \toprule
        \textbf{Dataset} & \textbf{Characteristic} & \textbf{Small} & \textbf{Medium} & \textbf{Large} \\
        \midrule
        \multirow{6}{*}{R-MNIST}
        & Conv Layers       & 1               & 2               & 2              \\
        & Conv Filters/Layer& (16)            & (24, 48)        & (32, 64)       \\
        & FC Layers         & 1               & 1               & 2              \\
        & Hidden FC Units   & None            & None            & 120            \\
        & Approx. Params\textsuperscript{*}    & $23\times10^3$      & $121\times10^3$      & $212\times10^3$    \\
        \midrule
        \multirow{6}{*}{H-CIFAR-10}
        & Conv Layers       & 1               & 2               & 2              \\
        & Conv Filters/Layer& (16)            & (24, 48)        & (32, 64)       \\
        & FC Layers         & 1               & 1               & 2              \\
        & Hidden FC Units   & None            & None            & 120            \\
        & Approx. Params\textsuperscript{*}    & $28\times10^3$    & $154\times10^3$    & $297\times10^3$   \\
        \midrule
        \multicolumn{2}{@{}l}{Client Allocation per Architecture} & $\approx N/3$ & $\approx N/3$ & $\approx N/3$ \\
        \bottomrule
        \multicolumn{5}{l}{\textsuperscript{*}\footnotesize{Approximate parameters (weights and biases) calculated based on standard inputs:}} \\
        \multicolumn{5}{l}{\footnotesize{~~R-MNIST ($28\times28\times1$), H-CIFAR-10 ($32\times32\times3$). Values rounded.}} \\
    \end{tabular}
\end{table}

\begin{figure}[t]
\centering
\begin{subfigure}[b]{\textwidth}
  \centering
  \includegraphics[scale=0.3]{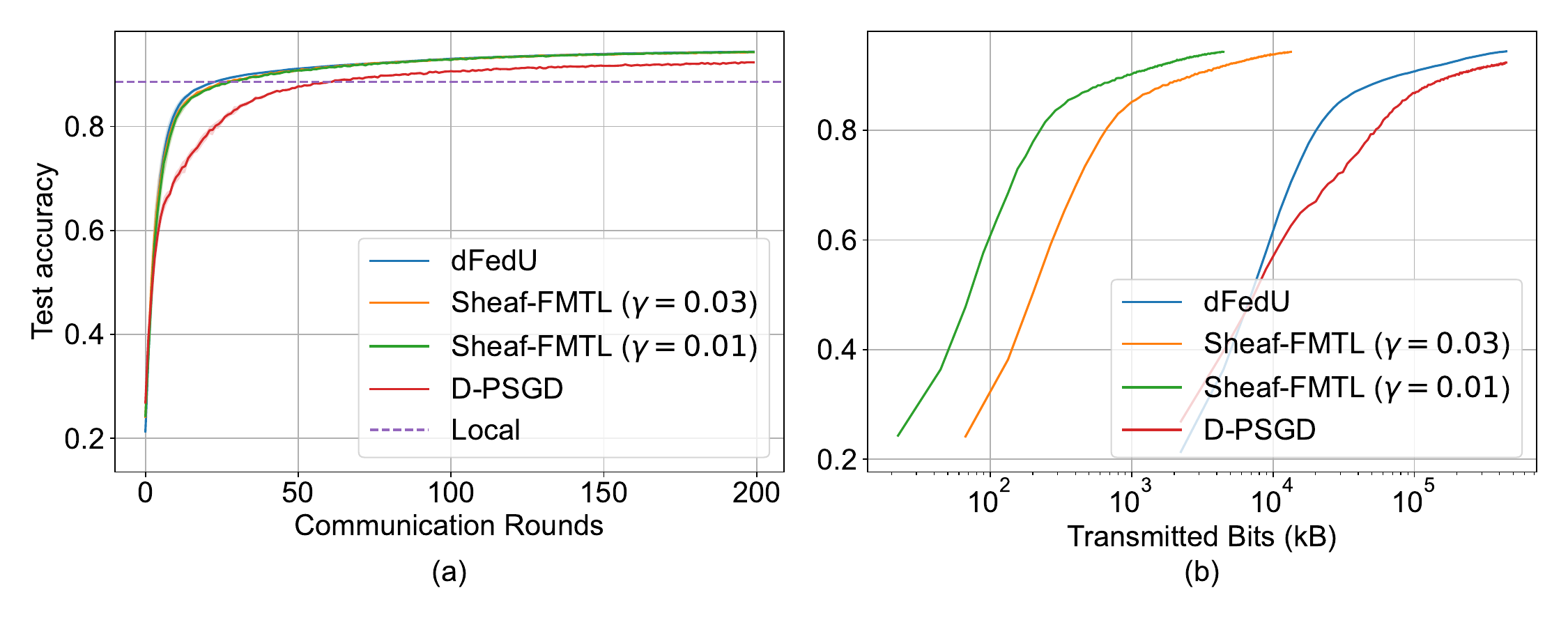}
\end{subfigure}%
\hfill
\begin{subfigure}[b]{\textwidth}
  \centering
  \includegraphics[scale=0.3]{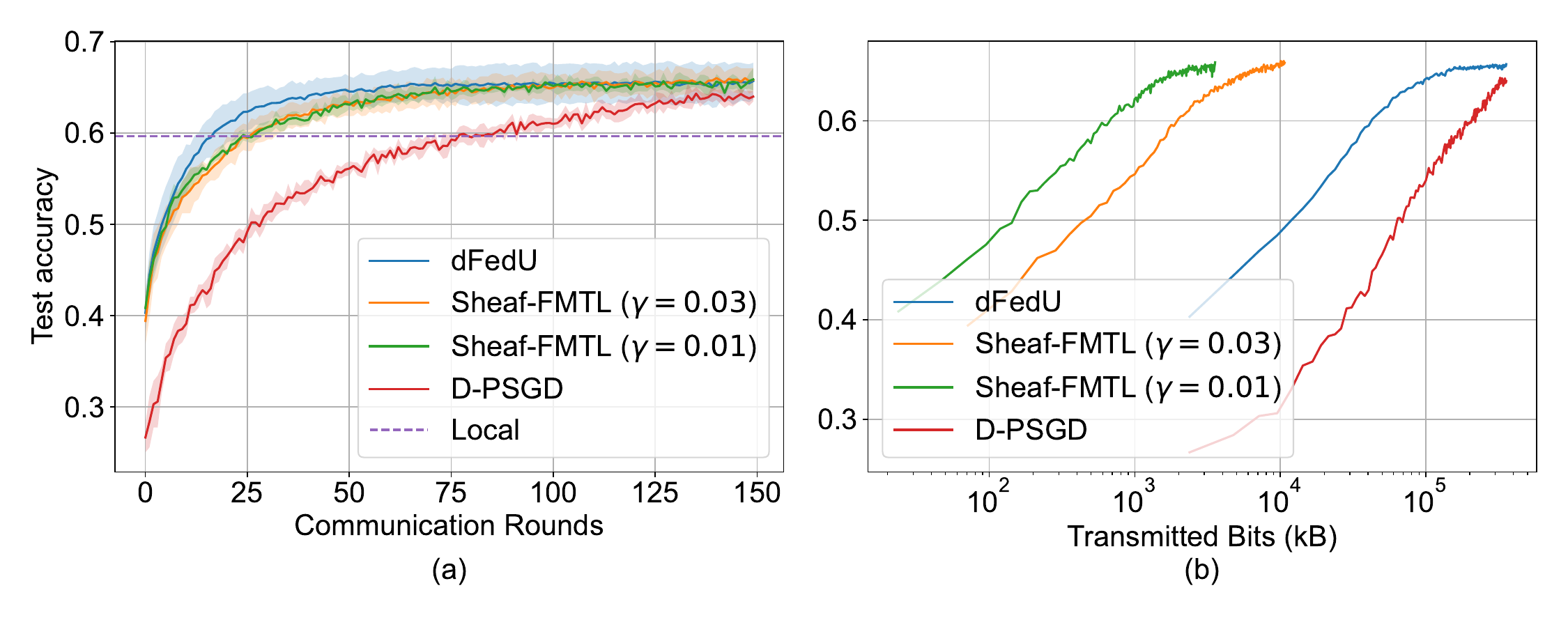}
\end{subfigure}%
\caption{Test accuracy as a function of the number of communication rounds and the number of transmitted bits for the rotated MNIST dataset (top), and the heterogeneous CIFAR-10 dataset (bottom).}
\label{har_school_datasets}
\end{figure}

\textbf{Baselines.} In the first experiment, we first compare our algorithm to the dFedU algorithm \citep{dinh2022new} as well as D-PSGD \citep{lian2017can}, and local training. Then, we provide more results by comparing {\ours} to state-of-the-art FL algorithms \citep{mcmahan2017communication, li2021ditto, fallah2020personalized, t2020personalized}. We implement our proposed algorithm using a mini-batch stochastic gradient in (\ref{thetaupdate}) to make the comparison fair. For dFedU, the weights $\{a_{ij}\}$, defined in (\ref{eqn: conventional FMTL objective}), are taken to be $a_{ij} = 1, ~\forall (i, j) \in \mathcal{E}$. In the second one, we compare to a stand-alone baseline where each client trains on each local dataset without communicating with the rest of the clients. To the best of our knowledge, {\ours} is the only algorithm solving the FMTL problem over decentralized topology with the clients having different model sizes, hence the comparison to a stand-alone baseline in the second experiment.\\ 
\textbf{Hardware and code.} Our experiments were carried out on a system equipped with an Intel(R) Xeon(R) CPU operating at 2.20GHz with 2 cores and 12 GB of RAM. The algorithms are implemented in Python using PyTorch \citep{paszke2019pytorch}, and NetworkX \citep{hagberg2008exploring}.
\subsection{Experiment 1: same model size}\label{experiment1}
Figure \ref{har_school_datasets} illustrates the performance of the proposed {\ours} algorithm and dFedU on the rotated MNIST and heterogeneous CIFAR-10 datasets, respectively, showcasing the test accuracy as a function of the number of communication rounds and the total number of transmitted bits. We consider two values for $\gamma = \{0.01, 0.03\}$ such that the projection space dimension is $\gamma d$. In Figure \ref{har_school_datasets}(a), we can see that {\ours} achieves similar test accuracies as dFedU. For both datasets, {\ours} ($\gamma = 0.01$) and dFedU converge significantly faster than D-PSGD and Local training. On rotated MNIST, {\ours} achieves 85\% accuracy within 15 rounds, while D-PSGD requires approximately twice the number of rounds to reach comparable performance. On rotated MNIST, {\ours} and dFedU both achieve final test accuracies above 94\%, outperforming D-PSGD ($\sim 89\%$) and Local training ($\sim 88\%$). The performance gap is even more pronounced on Heterogeneous CIFAR-10, where {\ours} and dFedU reach $\sim 66$ accuracy while D-PSGD struggles to exceed 55\%. On the other hand, Figure \ref{har_school_datasets}(b) shows that {\ours} achieves higher test accuracy with fewer transmitted bits compared to the baselines, demonstrating its ability to learn effectively while minimizing communication overhead. For the rotated MNIST dataset, using $\gamma = 0.01$ leads to almost similar test accuracy as dFedU, while it requires $100 \times$ less in terms of the number of transmitted bits to achieve this accuracy. For the Heterogeneous CIFAR-10, {\ours} requires slightly more communication rounds than dFedU, but the benefit in terms of communication overhead is evident as it requires exchanging significantly fewer bits. 

\begin{table}[t]
\centering
\caption{
  Comparative analysis of storage, computational overheads, and communication costs.
  }
\label{table:costs} 
\vspace{0.5em}
\begin{tabular}{@{}lccc@{}}
\toprule
\begin{tabular}[c]{@{}c@{}} \textbf{Method}\\ \textbf{ } \end{tabular} & \begin{tabular}[c]{@{}c@{}} \textbf{Storage}\\ \textbf{per client}\end{tabular} & \begin{tabular}[c]{@{}c@{}} \textbf{Compute per }\\ \textbf{iteration per client}\end{tabular} & \begin{tabular}[c]{@{}c@{}} \textbf{Communication per }\\ \textbf{iteration per client}\end{tabular}  \\ \midrule
{\ours} & $d_i + \sum\limits_{j \in \mathcal{N}_i} d_{ij} \times d_i$ & $\mathcal{O}(d_i \times d_{ij})$ & $\sum\limits_{j \in \mathcal{N}_i} d_{ij}$  \\ \midrule
dFedU &  $d_i$ & $\mathcal{O}(d_i)$ & $\sum\limits_{j \in \mathcal{N}_i} d_{j}$ \\ \bottomrule
\end{tabular}
\end{table}
\begin{figure}[t]
\centering
\includegraphics[scale=0.35]{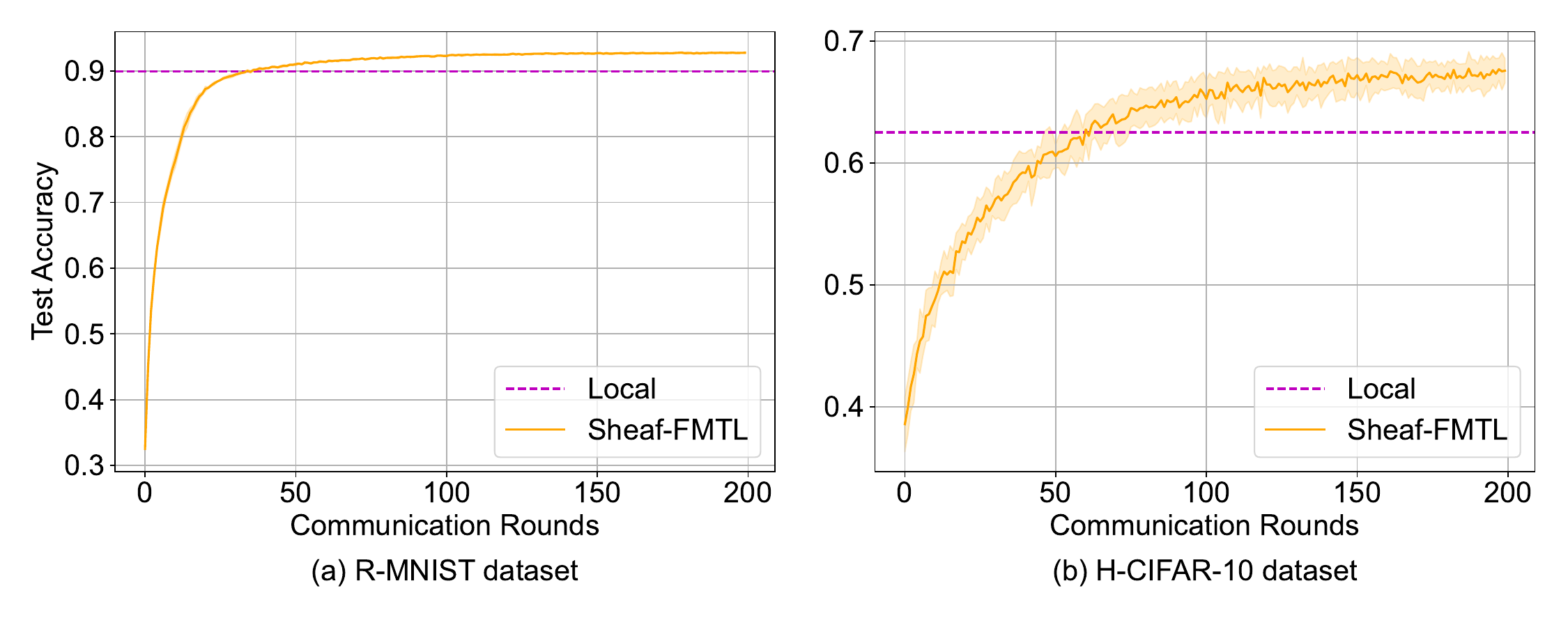}
\caption{Test accuracy as a function of the communication rounds for  (a) R-MNIST dataset and (b) H-CIFAR-10 dataset.}
\label{fig:experiment2}
\end{figure}

\begin{figure}[t]
\centering
\includegraphics[scale=0.35]{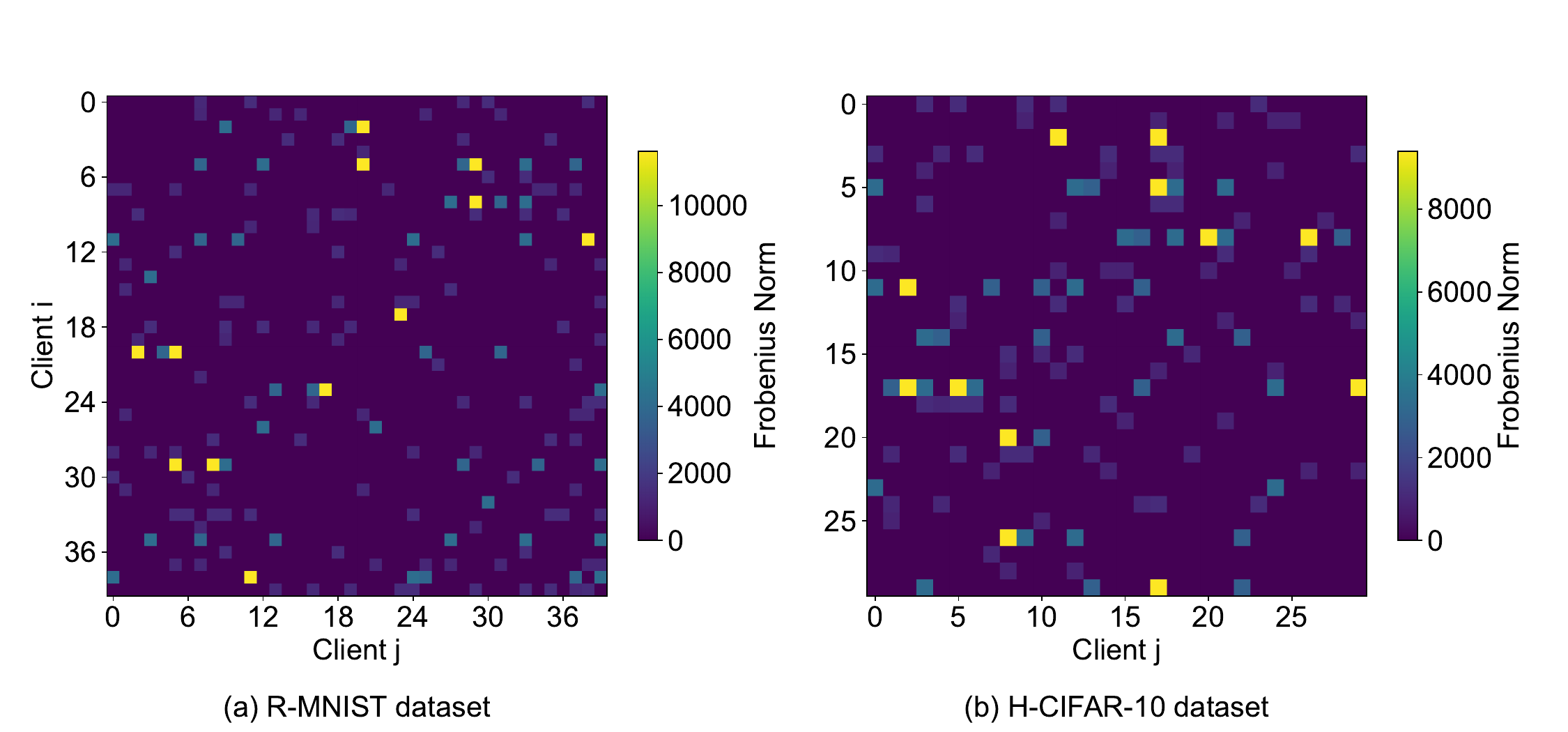}
\caption{Heat maps of the Frobenius norms of the final restriction maps for  (a) R-MNIST dataset and (b) H-CIFAR-10 dataset.}
\label{fig:experiment22}
\end{figure}

As illustrated in Table \ref{table:costs}, {\ours} incurs additional storage and computational costs due to the maintenance and training of restriction maps. Specifically, each restriction map requires storing a matrix of size $d_{ij} \times d_i$, leading to a cumulative storage requirement of $\mathcal{O}(|\mathcal{E}| \times d_{ij} \times d_i)$ across the network. Computationally, updating these maps involves matrix multiplications and gradient calculations, adding a complexity of $\mathcal{O}( d_{ij} \times d_i)$ per edge per iteration. However, these costs are significantly offset by the substantial communication savings achieved, particularly when $d_{ij}$ is chosen to be a small fraction of $d_i$, making {\ours} a viable option in resource-rich FL environments such as cross-silo FL settings. Furthermore, our analysis in Table \ref{table:new} shows that the actual memory footprint, while potentially higher than the simplest baselines, remains within reasonable limits compared to the baselines.

In addition to FedAvg and D-PSGD, we compare the performance of our approach to three state-of-the-art PFL algorithms: DITTO, pFedMe, and Per-FedAvg. These PFL algorithms are well-regarded in the literature for addressing heterogeneity in FL environments. To evaluate the performance of our method against these baselines, we conducted additional experiments on two datasets: Vehicle Sensor and HAR. Table \ref{table:new} measures the performance of various federated learning algorithms across four critical dimensions: test accuracy (model effectiveness), transmitted bits in megabytes (communication efficiency), floating-point operations in trillion (computational complexity), and memory usage in megabytes (resource requirements) on both the Vehicle Sensor and Human Activity Recognition datasets. The test accuracy is averaged over 5 runs, reporting both mean and standard deviation, FLOPS represents the total number of floating-point operations (in trillion) required for each algorithm to converge, and the memory column shows the peak memory consumption in megabytes during the training process. The measurement captures the maximum memory footprint. For the rotated MNIST dataset, {\ours}  achieves the highest test accuracy ($94.5\%$) regardless of the parameter $\gamma$, surpassing both decentralized algorithms like D-PSGD ($92.2\%$) and centralized approaches like pFedme ($90.05\%$). Notably, {\ours}  maintains this superior performance while requiring only 38MB of communication (at $\gamma=0.1$), which is about $100\times$ lower than dFedU (3230.9KB) and comparable to FedAvg. On the heterogeneous CIFAR-10 dataset, {\ours}  shows competitive accuracy ($65.8\%$ at $\gamma=0.3$) compared to dFedU ($65.7\%$) while requiring about $100\%$ less communication overhead. As expected, the PFL baselines (DITTO, pFedMe, Per-FedAvg) and FedAvg achieve lower communication overheads due to their star topology, leveraging a central parameter server. The computational efficiency of {\ours} is particularly evident when compared to personalization-focused algorithms such as pFedme, which demands approximately $24\times$ more floating-point operations on the heterogeneous CIFAR-10 dataset. This substantial improvement in the computation-communication trade-off can be attributed to our algorithm's ability to learn task relationships through flexible sheaf structures rather than enforcing rigid constraints across all clients. The parameter $\gamma$ effectively controls this trade-off, with higher values yielding marginal accuracy improvements at the cost of increased communication overhead. These results confirm that {\ours} successfully balances personalization, communication efficiency, and computational complexity in heterogeneous federated environments.
\subsection{Experiment 2: different model sizes}\label{experiment2}
Figure \ref{fig:experiment2} compares the performance of the proposed {\ours} algorithm with the local training, i.e., training each model independently without communication with other clients, on the rotated MNIST and heterogeneous CIFAR-10 datasets by plotting the test accuracy as a function of the number of communication rounds. The horizontal dashed line (Local baseline) in each subplot represents the converged performance achievable when clients train solely on their own data without any communication, serving as a reference for the efficacy of collaboration. For both datasets, clients train different CNN architectures as per Table \ref{tab:heterogeneous_cnn_architectures}. Subplot \ref{fig:experiment2}~(a) presents the results on the rotated MNIST dataset, where heterogeneity arises from concept shift (different rotations applied to client data). {\ours} demonstrates rapid convergence, surpassing the local baseline accuracy of approximately $90\%$ within the initial 40 communication rounds. The algorithm continues to improve, eventually plateauing at a significantly higher test accuracy of around $93\%$, showcasing the substantial advantage gained through collaborative learning facilitated by the sheaf mechanism in this setting.
Subplot \ref{fig:experiment2}~(b) displays the performance on the heterogeneous CIFAR-10 dataset, which exhibits label distribution skew and quantity skew. Here, the task is more challenging due to the increased heterogeneity. While convergence is slower compared to rotated MNIST, {\ours} consistently improves over communication rounds. It surpasses the local baseline accuracy by approximately $62.5\%$ around 60 rounds and reaches a final accuracy of $68\%$

To gain qualitative insights into the interaction structures learned by {\ours} in the presence of architectural heterogeneity, Figure \ref{fig:experiment22} visualizes the Frobenius norms ($\|\bm{P}_{ij}\|_F$) of the learned restriction maps for the rotated MNIST (left) and heterogeneous CIFAR-10 (right) datasets. These plots correspond to the performance results shown in Figure \ref{fig:experiment2}, where clients employed diverse CNN architectures as detailed in Table \ref{tab:heterogeneous_cnn_architectures}. Each pixel $(i, j)$ in the heatmaps represents the magnitude of the learned interaction map $\bm{P}_{ij}$ projecting the model parameters of client $i$ into the shared space with client $j$. Brighter colors indicate a larger norm, signifying a stronger learned coupling or influence between the projected representations of the clients' models. The heatmaps clearly show that the algorithm learns non-trivial interaction maps (non-zero norms) between connected clients. Crucially, the norms vary significantly across different client pairs. This demonstrates that {\ours} does not impose uniform interaction strengths but rather learns them. While many potential interactions exist (off-diagonal entries), the number of pairs exhibiting very high norms (bright yellow spots) is relatively sparse in both datasets. This suggests that the algorithm identifies and emphasizes connections between specific client pairs deemed most beneficial for the joint optimization objective, while maintaining weaker couplings between others.

\begin{table}[t]
    \centering
        \caption{Performance of {\ours} compared to baselines for the rotated MNIST and heterogeneous CIFAR-10 datasets.}
    \label{table:new}
    \vspace{0.5em}
    \begin{tabular}{@{}cccccc@{}}
        \toprule
        \textbf{Dataset} & \textbf{Algorithm} & \begin{tabular}[c]{@{}c@{}} \textbf{Test}\\ \textbf{Accuracy} \end{tabular} & \begin{tabular}[c]{@{}c@{}} \textbf{Transmitted}\\ \textbf{Bits (MB)} \end{tabular} & \begin{tabular}[c]{@{}c@{}} \textbf{FLOPS} \\ ($\times 10^{12}$) \end{tabular}& \begin{tabular}[c]{@{}c@{}} \textbf{Memory} \\ \textbf{(MB)} \end{tabular}\\
        \midrule
        Rotated MNIST & {\ours} ($\gamma$ = 0.01) & 94.3 ± 0.12  & 38.2 & 240.7 & 7563.5\\
        \cline{2-6}
        & {\ours} ($\gamma$ = 0.03) & 94.5 ± 0.08 & 165.7 & 249.3 & 24022.2\\
        \cline{2-6}
        & dFedU & 94.4 ± 0.06  & 3230.9 & 1184.4 & 1582.9 \\
        \cline{2-6}
        & D-PSGD & 92.22 ± 0.06  & 3230.9 & 236.8 & 1438.6\\
         \cline{2-6}
        & FedAvg & 81.4 ± 0.16 & 27.9 & 1184.1 & 1531\\
        \cline{2-6}
        & DITTO & 92.2 ± 0.11  & 27.9 & 473.8 & 1714.6 \\
        \cline{2-6}
        & pFedme & 90.05 ± 0.14 & 27.9 & 5923.6 & 1550.7\\
        \cline{2-6}
        & Per-FedAvg & 88.45 ± 0.11  & 27.9 & 202 & 1421\\
        \midrule
        Heterogeneous CIFAR-10 & {\ours} ($\gamma$ = 0.01) & 65.5 ± 1.2  & 41.4 & 282.2 & 9582.5\\
        \cline{2-6}
        & {\ours} ($\gamma$ = 0.03) & 65.8 ± 1.3 & 108.8 & 294.5 & 27452.2\\
        \cline{2-6}
        & dFedU & 65.7 ± 1.8 & 3937.4 & 1376.5 & 9197.6\\
         \cline{2-6}
         & D-PSGD & 64 ± 0.5 & 3937.4 & 207.5 & 8402.4\\
         \cline{2-6}
        & FedAvg & 60.88 ± 1.4 & 25.5 & 1040 & 8431.6\\
        \cline{2-6}
        & DITTO & 63.4 ± 1.1 & 25.5  & 413.4 & 8650.9  \\
        \cline{2-6}
        & pFedme & 62.5 ± 1.3 & 25.5  & 6886.2 & 8496.6 \\
        \cline{2-6}
        & Per-FedAvg & 61.25 ± 1.4  & 25.5  & 216 & 3480 \\
       \bottomrule
    \end{tabular}
\end{table}

\begin{figure}[t]
\centering
\begin{subfigure}[b]{0.3\textwidth}
  \centering
  \includegraphics[scale=0.25]{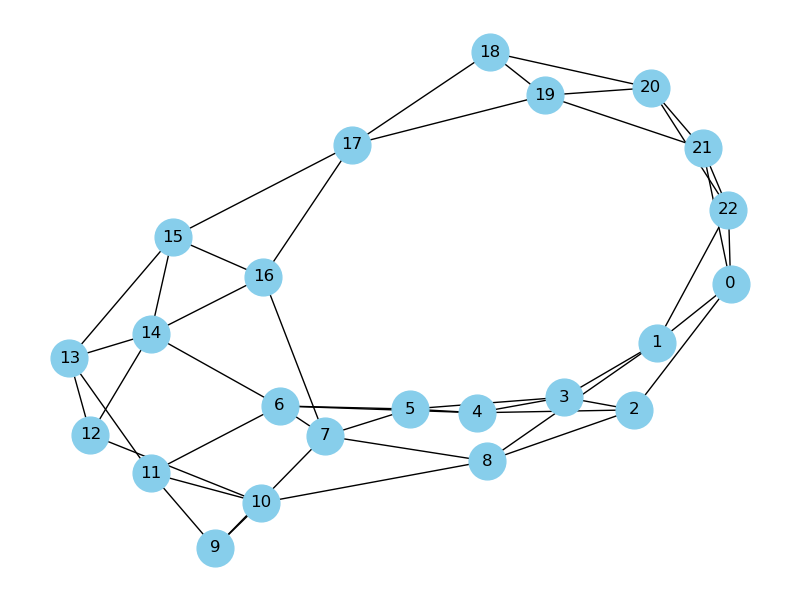}
  \caption{Small world}
  \label{small_world}
\end{subfigure}%
\hfill
\begin{subfigure}[b]{0.3\textwidth}
  \centering
  \includegraphics[scale=0.25]{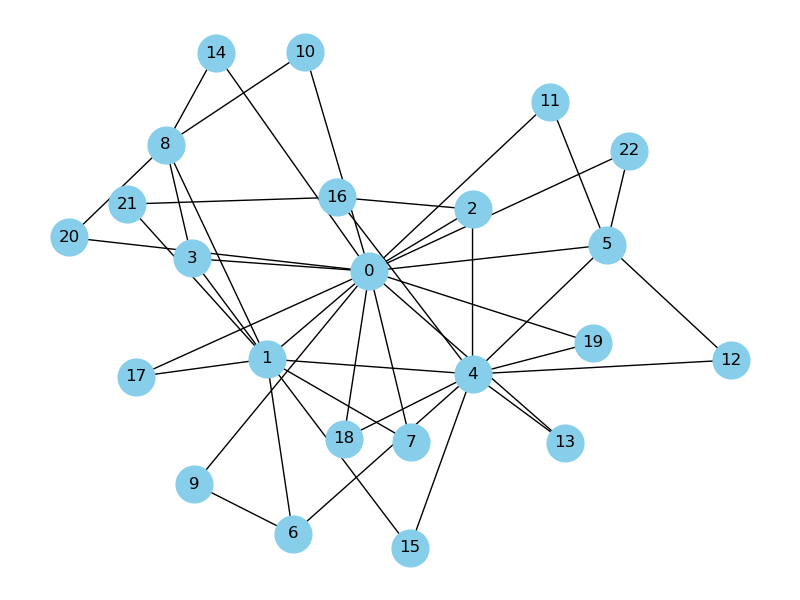}
  \caption{Scale free}
  \label{scale_free}
\end{subfigure}%
\hfill
\begin{subfigure}[b]{0.3\textwidth}
  \centering
  \includegraphics[scale=0.25]{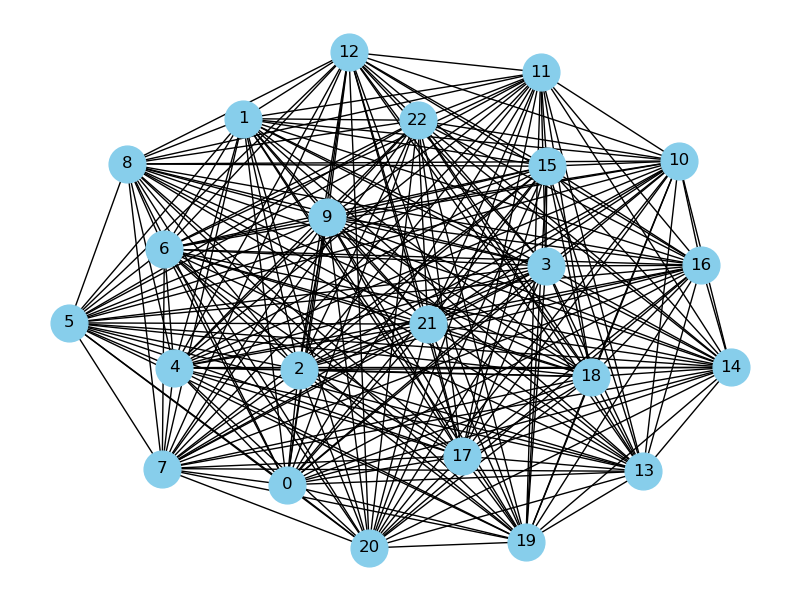}
  \caption{Complete}
  \label{complete}
\end{subfigure}%
\caption{Network topologies used in the regularization experiments: (a) Small-world topology generated using Watts-Strogatz model with $k=4$ neighbors and rewiring probability $p=0.1$, (b) Scale-free topology generated using Barabási-Albert model with $m=2$ attachments per new node, and (c) Complete graph where every client is connected to all others.}
\label{effect}
\end{figure}

\begin{table}[t]
\caption{Performance comparison of {\ours} and dFedU across different network topologies and regularization strengths on the Vehicle dataset.}
\label{tab:my-table}
\begin{tabular}{lccccc}
\toprule
\multirow{2}{*}{\textbf{Topology}} & \multirow{2}{*}{\textbf{\begin{tabular}[c]{@{}c@{}}Regularization \\ Strength\end{tabular}}} & \multicolumn{2}{c}{\textbf{\begin{tabular}[c]{@{}c@{}}{\ours}\\ ($\gamma = 0.1$)\end{tabular}}} & \multicolumn{2}{c}{\textbf{dFedU}} \\ \cmidrule{3-6} 
 &  & \textbf{\begin{tabular}[c]{@{}c@{}}Average Test\\ Accuracy\end{tabular}} & \textbf{\begin{tabular}[c]{@{}c@{}}Transmitted \\ Bits (kB)\end{tabular}} & \textbf{\begin{tabular}[c]{@{}c@{}}Average Test\\ Accuracy\end{tabular}} & \textbf{\begin{tabular}[c]{@{}c@{}}Transmitted \\ Bits (kB)\end{tabular}} \\ \midrule
\multirow{6}{*}{\begin{tabular}[c]{@{}c@{}}Small-world \\ (Fig. \ref{small_world})\end{tabular}} & $\lambda = 10^{-4}$ & 92.38 ± 0.38 & \multirow{6}{*}{162} & 91.79 ± 0.51 & \multirow{6}{*}{808} \\  \cmidrule{2-3} \cmidrule{5-5}
 & $\lambda = 10^{-3}$ & 92.31 ± 0.38 &  & 91.57 ± 0.36 &  \\ \cmidrule{2-3} \cmidrule{5-5}
 & $\lambda = 10^{-2}$ & 92.19 ± 0.40 &  & 91.89 ± 0.44 &  \\ \cmidrule{2-3} \cmidrule{5-5}
 & $\lambda = 10^{-1}$ & 91.63 ± 0.47 &  & 90.65 ± 0.58 &  \\ \cmidrule{2-3} \cmidrule{5-5}
 & $\lambda = 1$ & 89.80 ± 0.73 &  & 84.64 ± 1.72 &  \\ \cmidrule{2-3} \cmidrule{5-5}
 & $\lambda = 10$ & 84.15 ± 1.20 &  & 80.88 ± 0.91 &  \\ \midrule
\multirow{6}{*}{\begin{tabular}[c]{@{}c@{}}Scale-free \\ (Fig. \ref{scale_free})\end{tabular}} & $\lambda = 10^{-4}$ & 92.42 ± 0.40 & \multirow{6}{*}{485} & 91.58 ± 0.43 & \multirow{6}{*}{2424} \\ \cmidrule{2-3} \cmidrule{5-5}
 & $\lambda = 10^{-3}$ & 91.63 ± 0.48 &  & 91.28 ± 0.72 &  \\ \cmidrule{2-3} \cmidrule{5-5}
 & $\lambda = 10^{-2}$ & 92.45 ± 0.24 &  & 91.12 ± 0.20 &  \\ \cmidrule{2-3} \cmidrule{5-5}
 & $\lambda = 10^{-1}$ & 90.19 ± 0.55 &  & 90.54 ± 0.50 &  \\ \cmidrule{2-3} \cmidrule{5-5}
 & $\lambda = 1$ & 89.08 ± 0.75 &  & 85.71 ± 0.30 &  \\ \cmidrule{2-3} \cmidrule{5-5}
 & $\lambda = 10$ & 80.55 ± 1.34 &  & 80.41 ± 0.35 &  \\ \midrule
\multirow{6}{*}{\begin{tabular}[c]{@{}c@{}}Complete  \\ (Fig. \ref{complete})\end{tabular}} & $\lambda = 10^{-4}$ & 92.06 ± 0.67 & \multirow{6}{*}{711} & 88.98 ± 0.46& \multirow{6}{*}{3555} \\ \cmidrule{2-3} \cmidrule{5-5}
 & $\lambda = 10^{-3}$ & 91.06 ± 0.26 &  & 89.60 ± 0.61 &  \\ \cmidrule{2-3} \cmidrule{5-5}
 & $\lambda = 10^{-2}$ & 89.62 ± 0.52 &  & 89.40 ± 0.63 &  \\ \cmidrule{2-3} \cmidrule{5-5}
 & $\lambda = 10^{-1}$ & 89.15 ± 0.56 &  & 89.07 ± 0.34 &  \\ \cmidrule{2-3} \cmidrule{5-5}
 & $\lambda = 1$ & 84.5 ± 0.74&  & 83.67 ± 0.63 &  \\ \cmidrule{2-3} \cmidrule{5-5}
 & $\lambda = 10$ & 78.2 ± 0.45 &  & 80.40 ± 0.35&  \\ \bottomrule
\end{tabular}
\end{table}

\subsection{Ablation Study}\label{ablation}
\subsubsection{Effect of the Regularization Strength}\label{sec:reg}
To investigate the impact of both network topology and regularization strength on our proposed method, we conducted experiments on the Vehicle dataset using three distinct network structures as illustrated in Figure \ref{effect}: small-world, scale-free, and complete graphs. For the small-world topology, we used the Watts-Strogatz model with each node connected to $k=4$ nearest neighbors and a rewiring probability $p=0.1$, creating a network with high clustering and relatively short path lengths. The scale-free network was generated using the Barabási-Albert preferential attachment model with $m=2$ new edges per node, resulting in a power-law degree distribution where some nodes act as hubs with many connections. The regularization parameter $\lambda$ was varied across six orders of magnitude from $10^{-4}$ to $10$ to observe its effect on model performance.

Table \ref{tab:my-table} presents a comprehensive comparison between {\ours} and dFedU across these different configurations. Several key observations emerge from these results. First, {\ours} consistently achieves superior communication efficiency, requiring approximately $5\times$ fewer transmitted bits than dFedU across all topologies. Second, regarding accuracy performance, {\ours} generally outperforms dFedU across most regularization values, with the performance gap becoming more pronounced at higher regularization strengths ($\lambda \geq 1$). For instance, with $\lambda = 1$ in the small-world topology, {\ours} achieves 89.80\% accuracy compared to dFedU's 84.64\%. Among the three topologies, the small-world network provides the best balance between communication efficiency and model performance, requiring the least communication while maintaining comparable or better accuracy than more densely connected topologies.

While our experiments on the Vehicle dataset indicated optimal performance for $\lambda$ in the range $[10^{-4}, 10^{-2}]$ across different topologies, this should not be interpreted as a universal recommendation. The optimal choice of $\lambda$ is highly dependent on factors such as the degree and type of data heterogeneity, network topology, and the data and model size. In practice, $\lambda$ should be treated as a hyperparameter tuned using a validation set (typically a held-out portion of each client's local data). We recommend searching over a logarithmic scale (e.g.,  $[10^{-5}, 10]$), noting that excessively large values ($\lambda \geq 1$) can overly constrain personalization and harm performance.

\subsubsection{Effect of the Restriction Maps Initialization}\label{sec:maps}

\begin{table}[t]
\centering
\caption{Final average test accuracy of {\ours} for different initialization strategies of the restriction maps $\bm{P}_{ij}$ under heterogeneous CNN architectures (Table~\ref{tab:heterogeneous_cnn_architectures}).}
\label{tab:initialization_ablation}
\sisetup{table-format=1.4} 
\begin{tabular}{@{}l S[table-format=1.2, table-alignment=center] S S@{}} 
\toprule
\textbf{Initialization Method} & {\textbf{Parameter}} & {\textbf{R-MNIST}} & {\textbf{H-CIFAR-10}} \\
\midrule
\multirow{3}{*}{Gaussian~$\mathcal{N}(0, \sigma^2)$}
 & {$\sigma = 0.01$} & 0.9280 & 0.6767 \\
 & {$\sigma = 0.1$}  & 0.9279 & 0.6722 \\
 & {$\sigma = 1.0$}  & 0.9277 & 0.6815 \\
\midrule
\multirow{3}{*}{Uniform~ $\mathcal{U}([-a, a])$}
 & {$a = 0.01$} & \bfseries 0.9288 & \bfseries 0.6842 \\ 
 & {$a = 0.1$}   & 0.9268 & 0.6759 \\
 & {$a = 1$}       & 0.9277 & 0.6733 \\
\midrule
Orthogonal & {N/A} & 0.9273 & 0.6746 \\
\midrule
Identity + Noise & {$\sigma_{noise}=0.01$} & 0.9268 & 0.6624 \\
\bottomrule
\end{tabular}
\end{table}

We conduct an ablation study to assess the sensitivity of the {\ours} algorithm to the initialization strategy of the restriction maps $\bm{P}_{ij}$, particularly within the challenging context of heterogeneous model architectures across clients. Using the CNN architectures specified in Table \ref{tab:heterogeneous_cnn_architectures} for rotated MNIST and heterogeneous CIFAR-10 datasets, we evaluated several initialization methods
\begin{itemize}
    \item \textbf{Gaussian Initialization:} Each element of $\bm{P}_{ij}$ is drawn independently from a normal distribution $\mathcal{N}(0, \sigma^2)$, with variances $\sigma \in \{0.01, 0.1, 1\}$.
    \item \textbf{Uniform Initialization:} Each element was drawn independently from a uniform distribution $\mathcal{U}([-a, a])$, with ranges corresponding to $a \in \{0.01, 0.1, 1.0\}$.
    \item \textbf{Orthogonal Initialization:} $\bm{P}_{ij}$ was initialized as a matrix with orthonormal rows, ensuring $\bm{P}_{ij} \bm{P}_{ij}^T = \bm{I}_{d_{ij}}$.
    \item \textbf{Identity + Noise Initialization:} $\bm{P}_{ij}$ was initialized based on the identity matrix (selecting the first $d_{ij}$ rows of the $d_i \times d_i$ identity, padded with zeros if needed), with small Gaussian noise ($\mathcal{N}(0, 0.01^2)$) added to break perfect symmetry and allow adaptation. 
\end{itemize}
The results, presented in Table~\ref{tab:initialization_ablation}, measure the final average test accuracy achieved after convergence.

On the rotated MNIST dataset, the performance of {\ours} exhibits remarkable robustness to the choice of initialization. Across all tested strategies, the final accuracy remains tightly clustered between approximately $92.68\%$ and $92.88\%$. The best performance was marginally achieved with a narrow Uniform initialization ($\mathcal{U}([-0.01, 0.01])$). Still, the minimal variation suggests that the algorithm effectively converges to a high-quality solution regardless of the starting point for $\bm{P}_{ij}$ in this setting.

For the heterogeneous CIFAR-10 dataset, which presents greater data and architectural heterogeneity, we observe slightly more variation, though the overall stability is still noteworthy. The final accuracy ranges from $66.24\%$ to $68.42\%$. Similar to rotated MNIST, the narrow Uniform initialization ($\mathcal{U}([-0.01, 0.01])$) yielded the highest accuracy. Notably, the (Identity + Noise) initialization resulted in the lowest accuracy ($66.24\%$), suggesting that starting with projections aligned to canonical parameter subsets might be less effective than random or orthogonal initializations when dealing with complex models and diverse data distributions inherent in the heterogeneous CIFAR-10.

\section{Conclusion} \label{conclusion}
In this work, we introduced a novel sheaf-based framework for FMTL that effectively tackles challenges arising from data and sample heterogeneity across clients. By leveraging cellular sheaves, our approach can flexibly model complex interactions between client models, even with varying feature spaces and model sizes. While {\ours} introduces additional memory and computational overhead for managing the sheaf structure, this is often offset by significant communication savings, the ability to handle heterogeneous model architectures and providing a unified view subsuming various existing FL methods, making it a compelling approach, particularly in cross-silo FL settings. Theoretically, we analyzed the convergence properties of {\ours}, establishing a sublinear convergence rate in line with state-of-the-art decentralized FMTL algorithms. Empirically, extensive experiments on benchmark datasets demonstrated the communication savings of {\ours} compared to the state-of-the-art baselines.

\bibliography{main}

\begin{thebibliography}{57}
\providecommand{\natexlab}[1]{#1}
\providecommand{\url}[1]{\texttt{#1}}
\expandafter\ifx\csname urlstyle\endcsname\relax
  \providecommand{\doi}[1]{doi: #1}\else
  \providecommand{\doi}{doi: \begingroup \urlstyle{rm}\Url}\fi

\bibitem[Anguita et~al.(2013)Anguita, Ghio, Oneto, Parra, Reyes-Ortiz, et~al.]{anguita2013public}
Davide Anguita, Alessandro Ghio, Luca Oneto, Xavier Parra, Jorge~Luis Reyes-Ortiz, et~al.
\newblock A public domain dataset for human activity recognition using smartphones.
\newblock In \emph{Esann}, volume~3, pp.\ ~3, 2013.

\bibitem[Antunes et~al.(2022)Antunes, Andr{\'e}~da Costa, K{\"u}derle, Yari, and Eskofier]{antunes2022federated}
Rodolfo~Stoffel Antunes, Cristiano Andr{\'e}~da Costa, Arne K{\"u}derle, Imrana~Abdullahi Yari, and Bj{\"o}rn Eskofier.
\newblock Federated learning for healthcare: Systematic review and architecture proposal.
\newblock \emph{ACM Transactions on Intelligent Systems and Technology (TIST)}, 13\penalty0 (4):\penalty0 1--23, 2022.

\bibitem[Arivazhagan et~al.(2019)Arivazhagan, Aggarwal, Singh, and Choudhary]{arivazhagan2019federated}
Manoj~Ghuhan Arivazhagan, Vinay Aggarwal, Aaditya~Kumar Singh, and Sunav Choudhary.
\newblock Federated learning with personalization layers.
\newblock \emph{arXiv preprint arXiv:1912.00818}, 2019.

\bibitem[Chen et~al.(2018)Chen, Luo, Dong, Li, and He]{chen2018federated}
Fei Chen, Mi~Luo, Zhenhua Dong, Zhenguo Li, and Xiuqiang He.
\newblock Federated meta-learning with fast convergence and efficient communication.
\newblock \emph{arXiv preprint arXiv:1802.07876}, 2018.

\bibitem[Collins et~al.(2021)Collins, Hassani, Mokhtari, and Shakkottai]{collins2021exploiting}
Liam Collins, Hamed Hassani, Aryan Mokhtari, and Sanjay Shakkottai.
\newblock Exploiting shared representations for personalized federated learning.
\newblock In \emph{International conference on machine learning}, pp.\  2089--2099. PMLR, 2021.

\bibitem[Deng et~al.(2020{\natexlab{a}})Deng, Kamani, and Mahdavi]{deng2020adaptive}
Yuyang Deng, Mohammad~Mahdi Kamani, and Mehrdad Mahdavi.
\newblock Adaptive personalized federated learning.
\newblock \emph{arXiv preprint arXiv:2003.13461}, 2020{\natexlab{a}}.

\bibitem[Deng et~al.(2020{\natexlab{b}})Deng, Kamani, and Mahdavi]{deng2020distributionally}
Yuyang Deng, Mohammad~Mahdi Kamani, and Mehrdad Mahdavi.
\newblock Distributionally robust federated averaging.
\newblock \emph{Advances in neural information processing systems}, 33:\penalty0 15111--15122, 2020{\natexlab{b}}.

\bibitem[Deng et~al.(2023)Deng, Kamani, Mahdavinia, and Mahdavi]{deng2024distributed}
Yuyang Deng, Mohammad~Mahdi Kamani, Pouria Mahdavinia, and Mehrdad Mahdavi.
\newblock Distributed personalized empirical risk minimization.
\newblock \emph{Advances in Neural Information Processing Systems}, 36, 2023.

\bibitem[Dinh et~al.(2022)Dinh, Vu, Tran, Dao, and Zhang]{dinh2022new}
Canh~T Dinh, Tung~T Vu, Nguyen~H Tran, Minh~N Dao, and Hongyu Zhang.
\newblock A new look and convergence rate of federated multitask learning with laplacian regularization.
\newblock \emph{IEEE Transactions on Neural Networks and Learning Systems}, 2022.

\bibitem[Duan et~al.(2019)Duan, Liu, Chen, Tan, Ren, Qiao, and Liang]{duan2019astraea}
Moming Duan, Duo Liu, Xianzhang Chen, Yujuan Tan, Jinting Ren, Lei Qiao, and Liang Liang.
\newblock Astraea: Self-balancing federated learning for improving classification accuracy of mobile deep learning applications.
\newblock In \emph{2019 IEEE 37th international conference on computer design (ICCD)}, pp.\  246--254. IEEE, 2019.

\bibitem[Duarte \& Hu(2004)Duarte and Hu]{duarte2004vehicle}
Marco~F Duarte and Yu~Hen Hu.
\newblock Vehicle classification in distributed sensor networks.
\newblock \emph{Journal of Parallel and Distributed Computing}, 64\penalty0 (7):\penalty0 826--838, 2004.

\bibitem[Elgabli et~al.(2022)Elgabli, Issaid, Bedi, Rajawat, Bennis, and Aggarwal]{elgabli22a}
Anis Elgabli, Chaouki~Ben Issaid, Amrit~Singh Bedi, Ketan Rajawat, Mehdi Bennis, and Vaneet Aggarwal.
\newblock {F}ed{N}ew: A communication-efficient and privacy-preserving {N}ewton-type method for federated learning.
\newblock In \emph{Proceedings of the 39th International Conference on Machine Learning}, volume 162, pp.\  5861--5877, 2022.

\bibitem[Fallah et~al.(2020)Fallah, Mokhtari, and Ozdaglar]{fallah2020personalized}
Alireza Fallah, Aryan Mokhtari, and Asuman Ozdaglar.
\newblock Personalized federated learning with theoretical guarantees: A model-agnostic meta-learning approach.
\newblock \emph{Advances in Neural Information Processing Systems}, 33:\penalty0 3557--3568, 2020.

\bibitem[Goldstein(1991)]{goldstein1991multilevel}
Harvey Goldstein.
\newblock Multilevel modelling of survey data.
\newblock \emph{Journal of the Royal Statistical Society. Series D (The Statistician)}, 40\penalty0 (2):\penalty0 235--244, 1991.

\bibitem[Hagberg et~al.(2008)Hagberg, Swart, and S~Chult]{hagberg2008exploring}
Aric Hagberg, Pieter Swart, and Daniel S~Chult.
\newblock Exploring network structure, dynamics, and function using networkx.
\newblock Technical report, Los Alamos National Lab.(LANL), Los Alamos, NM (United States), 2008.

\bibitem[Hansen \& Ghrist(2019)Hansen and Ghrist]{hansen2019distributed}
Jakob Hansen and Robert Ghrist.
\newblock Distributed optimization with sheaf homological constraints.
\newblock In \emph{2019 57th Annual Allerton Conference on Communication, Control, and Computing (Allerton)}, pp.\  565--571. IEEE, 2019.

\bibitem[Hanzely \& Richt{\'a}rik(2020)Hanzely and Richt{\'a}rik]{hanzely2020federated}
Filip Hanzely and Peter Richt{\'a}rik.
\newblock Federated learning of a mixture of global and local models.
\newblock \emph{arXiv preprint arXiv:2002.05516}, 2020.

\bibitem[Hanzely et~al.(2020)Hanzely, Hanzely, Horv{\'a}th, and Richt{\'a}rik]{hanzely2020lower}
Filip Hanzely, Slavom{\'\i}r Hanzely, Samuel Horv{\'a}th, and Peter Richt{\'a}rik.
\newblock Lower bounds and optimal algorithms for personalized federated learning.
\newblock \emph{Advances in Neural Information Processing Systems}, 33:\penalty0 2304--2315, 2020.

\bibitem[Hsieh et~al.(2020)Hsieh, Phanishayee, Mutlu, and Gibbons]{hsieh2020non}
Kevin Hsieh, Amar Phanishayee, Onur Mutlu, and Phillip Gibbons.
\newblock The non-iid data quagmire of decentralized machine learning.
\newblock In \emph{International Conference on Machine Learning}, pp.\  4387--4398. PMLR, 2020.

\bibitem[Hsu et~al.(2019)Hsu, Qi, and Brown]{hsu2019measuring}
Tzu-Ming~Harry Hsu, Hang Qi, and Matthew Brown.
\newblock Measuring the effects of non-identical data distribution for federated visual classification.
\newblock \emph{arXiv preprint arXiv:1909.06335}, 2019.

\bibitem[Huang et~al.(2021)Huang, Chu, Zhou, Wang, Liu, Pei, and Zhang]{huang2021personalized}
Yutao Huang, Lingyang Chu, Zirui Zhou, Lanjun Wang, Jiangchuan Liu, Jian Pei, and Yong Zhang.
\newblock Personalized cross-silo federated learning on non-iid data.
\newblock In \emph{Proceedings of the AAAI conference on artificial intelligence}, volume~35, pp.\  7865--7873, 2021.

\bibitem[Jiang et~al.(2019)Jiang, Kone{\v{c}}n{\`y}, Rush, and Kannan]{jiang2019improving}
Yihan Jiang, Jakub Kone{\v{c}}n{\`y}, Keith Rush, and Sreeram Kannan.
\newblock Improving federated learning personalization via model agnostic meta learning.
\newblock \emph{arXiv preprint arXiv:1909.12488}, 2019.

\bibitem[Kairouz et~al.(2021)Kairouz, McMahan, Avent, Bellet, Bennis, Bhagoji, Bonawitz, Charles, Cormode, Cummings, et~al.]{kairouz2019advances}
Peter Kairouz, H~Brendan McMahan, Brendan Avent, Aur{\'e}lien Bellet, Mehdi Bennis, Arjun~Nitin Bhagoji, Kallista Bonawitz, Zachary Charles, Graham Cormode, Rachel Cummings, et~al.
\newblock Advances and open problems in federated learning.
\newblock \emph{Foundations and trends{\textregistered} in machine learning}, 14\penalty0 (1--2):\penalty0 1--210, 2021.

\bibitem[Karimireddy et~al.(2020)Karimireddy, Kale, Mohri, Reddi, Stich, and Suresh]{karimireddy2020scaffold}
Sai~Praneeth Karimireddy, Satyen Kale, Mehryar Mohri, Sashank Reddi, Sebastian Stich, and Ananda~Theertha Suresh.
\newblock Scaffold: Stochastic controlled averaging for federated learning.
\newblock In \emph{International conference on machine learning}, pp.\  5132--5143. PMLR, 2020.

\bibitem[Koh et~al.(2021)Koh, Sagawa, Marklund, Xie, Zhang, Balsubramani, Hu, Yasunaga, Phillips, Gao, et~al.]{koh2021wilds}
Pang~Wei Koh, Shiori Sagawa, Henrik Marklund, Sang~Michael Xie, Marvin Zhang, Akshay Balsubramani, Weihua Hu, Michihiro Yasunaga, Richard~Lanas Phillips, Irena Gao, et~al.
\newblock Wilds: A benchmark of in-the-wild distribution shifts.
\newblock In \emph{International conference on machine learning}, pp.\  5637--5664. PMLR, 2021.

\bibitem[Li et~al.(2022)Li, Diao, Chen, and He]{li2022federated}
Qinbin Li, Yiqun Diao, Quan Chen, and Bingsheng He.
\newblock Federated learning on non-iid data silos: An experimental study.
\newblock In \emph{2022 IEEE 38th international conference on data engineering (ICDE)}, pp.\  965--978. IEEE, 2022.

\bibitem[Li et~al.(2020{\natexlab{a}})Li, Sahu, Zaheer, Sanjabi, Talwalkar, and Smith]{li2020federated}
Tian Li, Anit~Kumar Sahu, Manzil Zaheer, Maziar Sanjabi, Ameet Talwalkar, and Virginia Smith.
\newblock Federated optimization in heterogeneous networks.
\newblock \emph{Proceedings of Machine learning and systems}, 2:\penalty0 429--450, 2020{\natexlab{a}}.

\bibitem[Li et~al.(2021{\natexlab{a}})Li, Hu, Beirami, and Smith]{li2021ditto}
Tian Li, Shengyuan Hu, Ahmad Beirami, and Virginia Smith.
\newblock Ditto: Fair and robust federated learning through personalization.
\newblock In \emph{International Conference on Machine Learning}, pp.\  6357--6368. PMLR, 2021{\natexlab{a}}.

\bibitem[Li et~al.(2020{\natexlab{b}})Li, Huang, Yang, Wang, and Zhang]{Li2020}
Xiang Li, Kaixuan Huang, Wenhao Yang, Shusen Wang, and Zhihua Zhang.
\newblock On the convergence of fedavg on non-iid data.
\newblock In \emph{International Conference on Learning Representations}, 2020{\natexlab{b}}.

\bibitem[Li et~al.(2021{\natexlab{b}})Li, JIANG, Zhang, Kamp, and Dou]{li2021model}
Xiaoxiao Li, Meirui JIANG, Xiaofei Zhang, Michael Kamp, and Qi~Dou.
\newblock Fed{BN}: Federated learning on non-{IID} features via local batch normalization.
\newblock In \emph{International Conference on Learning Representations}, 2021{\natexlab{b}}.
\newblock URL \url{https://openreview.net/forum?id=6YEQUn0QICG}.

\bibitem[Lian et~al.(2017)Lian, Zhang, Zhang, Hsieh, Zhang, and Liu]{lian2017can}
Xiangru Lian, Ce~Zhang, Huan Zhang, Cho-Jui Hsieh, Wei Zhang, and Ji~Liu.
\newblock Can decentralized algorithms outperform centralized algorithms? a case study for decentralized parallel stochastic gradient descent.
\newblock \emph{Advances in neural information processing systems}, 30, 2017.

\bibitem[Liang et~al.(2020)Liang, Liu, Ziyin, Allen, Auerbach, Brent, Salakhutdinov, and Morency]{liang2020think}
Paul~Pu Liang, Terrance Liu, Liu Ziyin, Nicholas~B Allen, Randy~P Auerbach, David Brent, Ruslan Salakhutdinov, and Louis-Philippe Morency.
\newblock Think locally, act globally: Federated learning with local and global representations.
\newblock \emph{arXiv preprint arXiv:2001.01523}, 2020.

\bibitem[Lin et~al.(2020)Lin, Stich, Patel, and Jaggi]{Lin2020}
Tao Lin, Sebastian~U. Stich, Kumar~Kshitij Patel, and Martin Jaggi.
\newblock Don't use large mini-batches, use local sgd.
\newblock In \emph{International Conference on Learning Representations}, 2020.

\bibitem[Liu et~al.(2022{\natexlab{a}})Liu, Hu, Wu, and Smith]{liu2022privacy}
Ken Liu, Shengyuan Hu, Steven~Z Wu, and Virginia Smith.
\newblock On privacy and personalization in cross-silo federated learning.
\newblock \emph{Advances in neural information processing systems}, 35:\penalty0 5925--5940, 2022{\natexlab{a}}.

\bibitem[Liu et~al.(2022{\natexlab{b}})Liu, Wu, Wu, Wang, Lyu, Chen, and Xie]{liu2022no}
Ruixuan Liu, Fangzhao Wu, Chuhan Wu, Yanlin Wang, Lingjuan Lyu, Hong Chen, and Xing Xie.
\newblock No one left behind: Inclusive federated learning over heterogeneous devices.
\newblock In \emph{Proceedings of the 28th ACM SIGKDD Conference on Knowledge Discovery and Data Mining}, pp.\  3398--3406, 2022{\natexlab{b}}.

\bibitem[Luo et~al.(2021)Luo, Chen, Hu, Zhang, Liang, and Feng]{luo2021no}
Mi~Luo, Fei Chen, Dapeng Hu, Yifan Zhang, Jian Liang, and Jiashi Feng.
\newblock No fear of heterogeneity: Classifier calibration for federated learning with non-iid data.
\newblock \emph{Advances in Neural Information Processing Systems}, 34:\penalty0 5972--5984, 2021.

\bibitem[McMahan et~al.(2017)McMahan, Moore, Ramage, Hampson, and y~Arcas]{mcmahan2017communication}
Brendan McMahan, Eider Moore, Daniel Ramage, Seth Hampson, and Blaise~Aguera y~Arcas.
\newblock Communication-efficient learning of deep networks from decentralized data.
\newblock In \emph{Artificial intelligence and statistics}, pp.\  1273--1282. PMLR, 2017.

\bibitem[Paszke et~al.(2019)Paszke, Gross, Massa, Lerer, Bradbury, Chanan, Killeen, Lin, Gimelshein, Antiga, et~al.]{paszke2019pytorch}
Adam Paszke, Sam Gross, Francisco Massa, Adam Lerer, James Bradbury, Gregory Chanan, Trevor Killeen, Zeming Lin, Natalia Gimelshein, Luca Antiga, et~al.
\newblock Pytorch: An imperative style, high-performance deep learning library.
\newblock \emph{Advances in neural information processing systems}, 32, 2019.

\bibitem[Quionero-Candela et~al.(2009)Quionero-Candela, Sugiyama, Schwaighofer, and Lawrence]{quionero2009dataset}
J.~Quionero-Candela, M.~Sugiyama, A.~Schwaighofer, and N.~D Lawrence.
\newblock \emph{Dataset shift in machine learning}.
\newblock The MIT Press, 2009.

\bibitem[Rahman et~al.(2015)Rahman, Merck, Huang, and Kleinberg]{rahman2015unintrusive}
Shah~Atiqur Rahman, Christopher Merck, Yuxiao Huang, and Samantha Kleinberg.
\newblock Unintrusive eating recognition using google glass.
\newblock In \emph{2015 9th International Conference on Pervasive Computing Technologies for Healthcare (PervasiveHealth)}, pp.\  108--111. IEEE, 2015.

\bibitem[Riess \& Ghrist(2022)Riess and Ghrist]{riess2022diffusion}
Hans Riess and Robert Ghrist.
\newblock Diffusion of information on networked lattices by gossip.
\newblock In \emph{2022 IEEE 61st Conference on Decision and Control (CDC)}, pp.\  5946--5952. IEEE, 2022.

\bibitem[Robinson(2013)]{robinson2013understanding}
Michael Robinson.
\newblock Understanding networks and their behaviors using sheaf theory.
\newblock In \emph{2013 IEEE Global Conference on Signal and Information Processing}, pp.\  911--914. IEEE, 2013.

\bibitem[Robinson(2014)]{robinson2014topological}
Michael Robinson.
\newblock Topological signal processing.
\newblock 81, 2014.

\bibitem[Robinson(2017)]{robinson2017sheaves}
Michael Robinson.
\newblock Sheaves are the canonical data structure for sensor integration.
\newblock \emph{Information Fusion}, 36:\penalty0 208--224, 2017.

\bibitem[SarcheshmehPour et~al.(2023)SarcheshmehPour, Tian, Zhang, and Jung]{sarcheshmehpour2023networked}
Yasmin SarcheshmehPour, Yu~Tian, Linli Zhang, and Alexander Jung.
\newblock Clustered federated learning via generalized total variation minimization.
\newblock \emph{IEEE Transactions on Signal Processing}, 71:\penalty0 4240--4256, 2023.

\bibitem[Sattler et~al.(2020)Sattler, M{\"u}ller, and Samek]{sattler2020clustered}
Felix Sattler, Klaus-Robert M{\"u}ller, and Wojciech Samek.
\newblock Clustered federated learning: Model-agnostic distributed multitask optimization under privacy constraints.
\newblock \emph{IEEE transactions on neural networks and learning systems}, 32\penalty0 (8):\penalty0 3710--3722, 2020.

\bibitem[Smith et~al.(2017)Smith, Chiang, Sanjabi, and Talwalkar]{smith2017federated}
Virginia Smith, Chao-Kai Chiang, Maziar Sanjabi, and Ameet~S Talwalkar.
\newblock Federated multi-task learning.
\newblock \emph{Advances in neural information processing systems}, 30, 2017.

\bibitem[T~Dinh et~al.(2020)T~Dinh, Tran, and Nguyen]{t2020personalized}
Canh T~Dinh, Nguyen Tran, and Josh Nguyen.
\newblock Personalized federated learning with moreau envelopes.
\newblock \emph{Advances in Neural Information Processing Systems}, 33:\penalty0 21394--21405, 2020.

\bibitem[Tan et~al.(2022)Tan, Long, Liu, Zhou, Lu, Jiang, and Zhang]{tan2022fedproto}
Yue Tan, Guodong Long, Lu~Liu, Tianyi Zhou, Qinghua Lu, Jing Jiang, and Chengqi Zhang.
\newblock Fedproto: Federated prototype learning across heterogeneous clients.
\newblock In \emph{Proceedings of the AAAI conference on artificial intelligence}, volume~36, pp.\  8432--8440, 2022.

\bibitem[Wang et~al.(2020{\natexlab{a}})Wang, Kaplan, Niu, and Li]{wang2020federated}
Hao Wang, Zakhary Kaplan, Di~Niu, and Baochun Li.
\newblock Optimizing federated learning on non-iid data with reinforcement learning.
\newblock In \emph{IEEE INFOCOM 2020-IEEE Conference on computer communications}, pp.\  1698--1707. IEEE, 2020{\natexlab{a}}.

\bibitem[Wang et~al.(2020{\natexlab{b}})Wang, Liu, Liang, Joshi, and Poor]{wang2021field}
Jianyu Wang, Qinghua Liu, Hao Liang, Gauri Joshi, and H~Vincent Poor.
\newblock Tackling the objective inconsistency problem in heterogeneous federated optimization.
\newblock \emph{Advances in neural information processing systems}, 33:\penalty0 7611--7623, 2020{\natexlab{b}}.

\bibitem[Wang et~al.(2019)Wang, Mathews, Kiddon, Eichner, Beaufays, and Ramage]{wang2019federated}
Kangkang Wang, Rajiv Mathews, Chlo{\'e} Kiddon, Hubert Eichner, Fran{\c{c}}oise Beaufays, and Daniel Ramage.
\newblock Federated evaluation of on-device personalization.
\newblock \emph{arXiv preprint arXiv:1910.10252}, 2019.

\bibitem[Ye et~al.(2020)Ye, Zhou, Luo, and Zhang]{ye2020decentralized}
Haishan Ye, Ziang Zhou, Luo Luo, and Tong Zhang.
\newblock Decentralized accelerated proximal gradient descent.
\newblock \emph{Advances in Neural Information Processing Systems}, 33:\penalty0 18308--18317, 2020.

\bibitem[Yu et~al.(2020)Yu, Bagdasaryan, and Shmatikov]{yu2020salvaging}
Tao Yu, Eugene Bagdasaryan, and Vitaly Shmatikov.
\newblock Salvaging federated learning by local adaptation.
\newblock \emph{arXiv preprint arXiv:2002.04758}, 2020.

\bibitem[Zhang et~al.(2024)Zhang, Yin, Hong, and Chen]{zhang2022hybrid}
Xinwei Zhang, Wotao Yin, Mingyi Hong, and Tianyi Chen.
\newblock Hybrid federated learning for feature \& sample heterogeneity: Algorithms and implementation.
\newblock \emph{Transactions on Machine Learning Research}, 2024.
\newblock ISSN 2835-8856.
\newblock URL \url{https://openreview.net/forum?id=qc2lmWkvk4}.

\bibitem[Zhao et~al.(2018)Zhao, Li, Lai, Suda, Civin, and Chandra]{zhao2018federated}
Yue Zhao, Meng Li, Liangzhen Lai, Naveen Suda, Damon Civin, and Vikas Chandra.
\newblock Federated learning with non-iid data.
\newblock \emph{arXiv preprint arXiv:1806.00582}, 2018.

\bibitem[Zhou et~al.(2011)Zhou, Chen, and Ye]{zhou2011malsar}
Jiayu Zhou, Jianhui Chen, and Jieping Ye.
\newblock Malsar: Multi-task learning via structural regularization.
\newblock \emph{Arizona State University}, 21:\penalty0 1--50, 2011.

\end{thebibliography}
\bibliographystyle{tmlr}

\newpage
\appendix
\section{Sheaf-Theoretic Approach in FMTL}
\label{app:sheaf_theory}
\subsection{Rationale for Adopting Sheaf Theory in FMTL}
\label{subsec:rationale_sheaf}
Sheaf theory provides a powerful mathematical framework for modelling and analyzing complex relationships in FMTL. The adoption of this approach in our context is motivated by several key advantages
\begin{enumerate}
    \item \textbf{Heterogeneity modeling.} FMTL often involves clients with different data distributions, model architectures, or task objectives. Sheaf theory allows us to capture these heterogeneous relationships in a structured and mathematically rigorous manner.
    \item \textbf{Local-Global consistency.} Sheaves provide a natural way to ensure consistency between local (client-specific) and global (network-wide) information. This is crucial in FMTL scenarios where we aim to leverage network information to improve local performance. 
    \item \textbf{Flexible representation.} The sheaf structure allows for representing varying degrees of similarity or difference between clients. This nuanced representation is more sophisticated than traditional approaches that often assume uniform relationships across the network.
\end{enumerate}
\subsection{The Interaction Space and Client Relationships}
\label{subsec:interaction_space}
The interaction space, denoted as $\mathcal{F}(e)$, plays a central role in our sheaf-theoretic approach to FMTL. It serves as a shared space where local models $\bm{\theta}_i$ and $\bm{\theta}_j$ are projected using restriction maps $\bm{P}_{ij}$ and $\bm{P}_{ji}$, respectively. The projection into this interaction space provides a measure of client relationships for the following reasons
\begin{enumerate}
    \item \textbf{Common feature capture.} The interaction space captures the common or comparable features between clients, analogous to how principal component analysis (PCA) captures the most important features of a dataset.
    \item \textbf{Consistency enforcement.} Our approach enforces consistency between the projections of local models onto the interaction space. This is mathematically represented by the sheaf Laplacian term $\frac{\lambda}{2} \bm{\theta}^T L_{\mathcal{F}}(\bm{P}) \bm{\theta}$, which penalizes discrepancies between the projections of local models.
    \item \textbf{Collaboration encouragement.} By minimizing the sheaf Laplacian term, local models are encouraged to collaborate effectively, leveraging shared information to improve overall performance.
\end{enumerate}
\subsection{Restriction Maps and Their Interpretations}
\label{subsec:restriction_maps}
The restriction maps, represented by matrices $\bm{P}_{ij}$, are fundamental to our sheaf-theoretic approach. These maps project local models $\bm{\theta}_i$ onto the interaction space $\mathcal{F}(e)$. The intuition behind these maps can be understood as follows
\begin{enumerate}
    \item \textbf{Feature selection.} $\bm{P}_{ij}$ acts as a feature selection matrix, identifying common or comparable features between clients $i$ and $j$.
    \item \textbf{Information sharing.} The restriction maps facilitate information sharing between clients by projecting local models onto a common space, enabling effective collaboration even when local models have different dimensions or feature sets.
    \item \textbf{Model comparison.} In heterogeneous settings where clients have different model sizes, traditional FL methods relying on direct model aggregation or comparison fail. The restriction maps allow for meaningful comparisons by projecting onto a common space.
\end{enumerate}
\subsection{Dimensional Considerations and Trade-offs}
\label{subsec:dimensions}
The dimensions of the restriction map $\bm{P}_{ij}$ are determined by the dimensions of the local model $\bm{\theta}_i$ and the interaction space $\mathcal{F}(e)$. If $\bm{\theta}_i \in \mathbb{R}^{d_i}$ and the interaction space has dimension $d_{ij}$, then $\bm{P}_{ij}$ is of size $d_{ij} \times d_i$. The choice of $d_{ij}$ involves a trade-off
\begin{itemize}
    \item \textbf{Smaller $d_{ij}$:} Results in a more compact representation of shared information, leading to communication savings.
    \item \textbf{Larger $d_{ij}$:} Allows for more flexibility in capturing relationships between local models but increases communication costs.
\end{itemize}
In practice, the choice of $d_{ij}$ can be guided by factors such as the estimated overlap in feature spaces between clients, computational resources available, and the desired balance between model expressiveness and communication efficiency.
\subsection{Comparative Advantages over Traditional FMTL Methods}
\label{subsec:advantages}
Our sheaf-theoretic approach offers several advantages over traditional FMTL methods
\begin{enumerate}
    \item \textbf{Heterogeneity handling.} Unlike many traditional FMTL methods that assume homogeneous models across clients, our approach naturally accommodates heterogeneous model architectures and task objectives.
    \item \textbf{Nuanced relationships.} Traditional methods often assume uniform relationships between clients. Our approach allows for more nuanced modelling of inter-client relationships through the interaction space and restriction maps.
\end{enumerate}
\section{Interpretation of $d_{ij}$ and $\textit{\textbf{P}}_{ij}$} \label{appendix:discussion}
In this appendix, we provide an interpretation of the restriction maps $\bm{P}_{ij} = \mathcal{F}_{i \trianglelefteq e}$ for the case when the local models are given by linear or logistic regression. We describe how to choose $d_{ij}$ and $\bm{P}_{ij}$ in a meaningful way and the constraints that can be imposed on them.

Consider a network of $N$ clients, where each client $i$ has a local model parameterized by $\bm{\theta}_i \in \mathbb{R}^{d_i}$. In the case of linear regression, the local model of client $i$ is given by $f_i(\bm{\theta}_i; \bm{x}_i) = \bm{\theta}_i^T \bm{x}_i$, where $\bm{x}_i \in \mathbb{R}^{d_i}$ is the input feature vector. For logistic regression, the local model is given by $f_i(\bm{\theta}_i; \bm{x}_i) = \sigma(\bm{\theta}_i^T \bm{x}_i)$, where $\sigma(\cdot)$ is the sigmoid function.

\textbf{Interpretation of $\bm{P}_{ij}$.} Consider two clients $i, j \in \mathcal{V}$ such that $e = (i, j) \in \mathcal{E}$. The restriction maps $\bm{P}_{ij}$ and $\bm{P}_{ji}$ aim to capture the relationships between the local models $\bm{\theta}_i$ and $\bm{\theta}_j$ by projecting them to a common interaction space $\mathcal{F}(e)$. In the context of linear or logistic regression, $\bm{P}_{ij}$ and $\bm{P}_{ji}$ can be interpreted as feature selection matrices that identify the common or comparable features between the two clients.

Let $\bm{P}_{ij} \in \mathbb{R}^{d{ij} \times d_i}$ and $\bm{P}_{ji} \in \mathbb{R}^{d{ij} \times d_j}$ be the restriction maps for clients $i$ and $j$, respectively, where $d_{ij} = \dim \mathcal{F}(e)$ is the dimension of the interaction space. The restriction maps should satisfy the following condition
\begin{align}
\bm{P}_{ij} \bm{\theta}_i \approx \bm{P}_{ji} \bm{\theta}_j,
\label{eqn: projection of models}
\end{align}
which ensures that the projected models in the interaction space are comparable.

To impose the condition in (\ref{eqn: projection of models}), we consider the following regularizer term
\begin{align}
Q_{ij} = \left\lVert \bm{P}_{ij} \bm{\theta}_i - \bm{P}_{ji} \bm{\theta}_j \right\rVert^2,
\label{eqn: regularizer}
\end{align}
which we aim to minimize. By adding the regularizer terms for all neighboring clients and then summing over all clients, we obtain the overall regularizer
\begin{align}
Q(\bm{\theta}) = \sum_{i=1}^N \sum_{j \in \mathcal{N}_i} \left\lVert \bm{P}_{ij} \bm{\theta}_i - \bm{P}_{ji} \bm{\theta}_j \right\rVert^2,
\label{eqn: overall regularizer}
\end{align}
which is exactly the sheaf quadratic form for the choice $\mathcal{F}_{i \trianglelefteq e} = \bm{P}_{ij}$ for $e = (i, j) \in \mathcal{E}$.\\
\textbf{Choice of $d_{ij}$.} The dimension of the interaction space, $d_{ij}$, determines the size of the restriction maps $\bm{P}_{ij}$ and $\bm{P}_{ji}$. In practice, $d_{ij}$ can be chosen based on the number of common or comparable features between clients $i$ and $j$. A smaller value of $d_{ij}$ implies a more compact representation of the shared information between the two clients, while a larger value allows for more flexibility in capturing the relationships between the local models.

\section{Real-World Applications of {\ours}}\label{appendix:vector_space_scenarios}
In this appendix, we explore practical scenarios where task similarities are inherently defined within vector spaces, making them well-suited for the application of {\ours}. By examining representation learning and feature vector similarities in multi-modal tasks, we illustrate how our sheaf-theoretic framework effectively captures and leverages complex task relationships. These examples showcase the applicability of {\ours} in diverse FL environments.
\subsection{Representation Learning and Embedding Spaces}
In many ML applications, tasks are associated with high-dimensional data that can be effectively represented through embeddings in vector spaces. These embeddings capture semantic, syntactic, or feature-based relationships between tasks, facilitating the modeling of task similarities as vector operations. For example, if we consider the Natural Language Processing (NLP) field, then tasks such as sentiment analysis, topic classification, and named entity recognition can be embedded in a semantic space using techniques like Word2Vec or BERT. The proximity of these task vectors in the embedding space reflects their semantic relatedness. Similar NLP tasks often share underlying linguistic structures. By representing these tasks as vectors, {\ours} can capture and leverage their shared characteristics to enhance collaborative learning.
\subsection{Feature Vector Similarities in Multi-Modal Tasks}
In multi-modal learning scenarios, tasks often involve integrating and processing data from different modalities (e.g., text, image, audio). Task similarities can be defined based on the feature vectors extracted from these modalities, enabling {\ours} to model interactions across diverse data sources. For example, if we consider multi-modal sentiment analysis, then tasks that analyze sentiment from text, images, and audio can have their respective feature vectors embedded in a unified vector space. The similarities between these feature vectors can indicate shared sentiment characteristics across modalities. {\ours} can utilize these vector similarities to facilitate collaborative learning, enhancing sentiment detection accuracy by leveraging cross-modal information. Another example is healthcare applications, where tasks involving the analysis of patient data from various sources (e.g., medical imaging, electronic health records, genomic data) can define task similarities based on the integrated feature vectors representing different data modalities. By modeling these similarities in vector space, {\ours} can improve personalized treatment recommendations through effective knowledge sharing across related healthcare tasks.
\section{Supporting lemmas}\label{suppLemmas}
\begin{lemma}
\begin{align}
    \bm{\theta}^T L_\F \, \bm{\theta} = \sum_{e=(i,j) \in \mathcal{E}} \left\lVert \F_{i\trianglelefteq e}\left( \bm{\theta}_i \right)  -  \F_{j\trianglelefteq e}\left( \bm{\theta}_j \right) \right\rVert^2 = \sum_{e=(i,j) \in \mathcal{E}} \left\| \bm{P}_{ij} \bm{\theta}_{i} - \bm{P}_{ji} \bm{\theta}_{j} \right\|^2.
\end{align}
\end{lemma}
\begin{proof}
Using the block matrix structure of $L_\F$ from equation (\ref{eqn: Laplacian_matrix_form}), we can expand the quadratic form as follows
\begin{align}
\nonumber &\bm{\theta}^T L_\F \, \bm{\theta}\\
\nonumber &= \sum_{i \in \mathcal{V}} \sum_{j \in \mathcal{V}} \bm{\theta}_i^T L_{i,j} \bm{\theta}_j \\
\nonumber &= \sum_{i \in \mathcal{V}} \bm{\theta}_i^T \left( \sum_{j \in \mathcal{N}_i} \bm{P}_{ij}^T \bm{P}_{ij} \right) \bm{\theta}_i - 2 \sum_{e=(i,j) \in \mathcal{E}} \bm{\theta}_i^T \bm{P}_{ij}^T \bm{P}_{ji} \bm{\theta}_j \\
\nonumber &= \sum_{e=(i,j) \in \mathcal{E}} \left( \bm{\theta}_i^T \bm{P}_{ij}^T \bm{P}_{ij} \bm{\theta}_i + \bm{\theta}_j^T \bm{P}_{ji}^T \bm{P}_{ji} \bm{\theta}_j - 2 \bm{\theta}_i^T \bm{P}_{ij}^T \bm{P}_{ji} \bm{\theta}_j \right) \\
\nonumber &= \sum_{e=(i,j) \in \mathcal{E}} \left( \left\| \bm{P}_{ij} \bm{\theta}_i \right\|^2 + \left\| \bm{P}_{ji} \bm{\theta}_j \right\|^2 - 2 \left\langle \bm{P}_{ij} \bm{\theta}_i, \bm{P}_{ji} \bm{\theta}_j \right\rangle \right) \\
&= \sum_{e=(i,j) \in \mathcal{E}} \left\| \bm{P}_{ij} \bm{\theta}_i - \bm{P}_{ji} \bm{\theta}_j \right\|^2.
\end{align}
Recalling that $\F_{i\trianglelefteq e}$ and $\bm{P}_{ij}$ are used interchangeably to denote the restriction map from vertex $i$ to edge $e=(i,j)$, we obtain
\begin{align}
\bm{\theta}^T L_\F \bm{\theta} = \sum_{e=(i,j) \in \mathcal{E}} \left\| \bm{P}_{ij} \bm{\theta}_i - \bm{P}_{ji} \bm{\theta}_j \right\|^2 = \sum_{e=(i,j) \in \mathcal{E}} \left\| \F_{i\trianglelefteq e}\left( \bm{\theta}_i \right) - \F_{j\trianglelefteq e}\left( \bm{\theta}_j \right) \right\|^2.
\end{align}
\end{proof}
\begin{lemma}
\begin{align}
\ker(L_\F) = \argmin_{\bm{\theta} \in  C^0(\F) }  Q_\F(\bm{\theta}) = \argmin_{\bm{\theta}\in  C^0(\F) } \bm{\theta}^T L_\F\, \bm{\theta}.    
\end{align}     
\end{lemma}
\begin{proof}
Let $\bm{\theta} \in \ker(L_\F)$. Then, $L_\F \bm{\theta} = 0$. By the definition of $Q_\F(\bm{\theta})$, we have
\begin{align*}
Q_\F(\bm{\theta}) = \bm{\theta}^T L_\F \bm{\theta} = 0.
\end{align*}
Therefore, $\bm{\theta} \in \argmin_{\bm{\theta} \in C^0(\F)} Q_\F(\bm{\theta})$.

Conversely, let $\bm{\theta} \in \argmin_{\bm{\theta} \in C^0(\F)} Q_\F(\bm{\theta})$. Then, $Q_\F(\bm{\theta}) = 0$, which implies
\begin{align}
0 = Q_\F(\bm{\theta}) = \sum_{e=(i,j) \in \mathcal{E}} \left\| \F_{i\trianglelefteq e}\left( \bm{\theta}_i \right) - \F_{j\trianglelefteq e}\left( \bm{\theta}_j \right) \right\|^2.
\end{align}
Thus, we get
\begin{align}
     \F_{i\trianglelefteq e}\left( \bm{\theta}_i \right) = \F_{j\trianglelefteq e}\left( \bm{\theta}_j \right) \quad \forall e=(i,j) \in \mathcal{E}.
\end{align}
This concludes the proof.    
\end{proof}
\section{Proof of Theorem \ref{theorem}}\label{theorem_proof}
We start by introducing the matrices $\bm{J}_{ij} \in \mathbb{R}^{d_{ij} \times d_i}$ having all its entries equal to one. Then, let us define the block matrix $\bm{H}$ such that $\bm{H}_{ij} = \bm{J}_{ij}$ if $(i, j) \in  \mathcal{E}$, and $\bm{H}_{ij} = \bm{0}$, otherwise. Then, $\bm{\theta}$ and $\bm{P}$ can be updated using the following updates 
\begin{align}
   \bm{\theta}^{k+1} &= \bm{\theta}^k- \alpha 
   (\nabla f(\bm{\theta}^k) + \lambda (\bm{P}^k)^T \bm{P}^k \bm{\theta}^k), \\
   \bm{P}^{k+1} &= \bm{H} \odot \left(\bm{P}^k- \eta \lambda \bm{P}^k \bm{\theta}^{k+1} (\bm{\theta}^{k+1})^T\right),
\end{align}
where $\odot$ is the Hadamard product and the matrix $\bm{H}$ is introduced to preserve the block structure of $\bm{P}$ by zeroing out the entries that do not correspond to edges in the graph.

Next, we analyze the convergence of the {\ours} algorithm by studying the descent steps in $\bm{\theta}$ and $\bm{P}$ separately. This approach allows us to establish bounds on the decrease of the objective function $\Psi(\bm{\theta}, \bm{P})$ in each descent step.

\begin{theorem}
Let Assumptions 1 and 2 hold. Assume the learning rates $\alpha$ and $\eta$ satisfy the conditions $\alpha < \frac{2}{N L}$ and $\eta < \frac{2}{\lambda D_\theta^2}$, respectively. Then, the averaged gradient norm is upper bounded as follows
\begin{align}
\frac{1}{K}\sum_{k=0}^{K-1} \|\nabla \Psi(\bm{\theta}^k, \bm{P}^k)\|^2 \leq \frac{1}{\rho K} (\Psi(\bm{\theta}^{0}, \bm{P}^0) - \Psi^\star),
\end{align}
where $\rho = \min\left\{\alpha \left(1 - \frac{\alpha N L}{2}\right), \eta  \left(1 - \frac{\eta \lambda D_{\theta}^2}{2} \right) \right\}$ and $\Psi^\star$ is the optimal value of $\Psi$.
\end{theorem}

\begin{proof}
Using the Lipschitz continuity of the gradient of $f(\bm{\theta})$
\begin{align}
f(\bm{\theta}^{k+1}) \leq f(\bm{\theta}^{k}) + \langle \nabla f(\bm{\theta}^{k}), \bm{\theta}^{k+1} - \bm{\theta}^{k} \rangle + \frac{N L}{2} \|\bm{\theta}^{k+1} - \bm{\theta}^{k}\|^2.    
\end{align}
Adding and subtracting $\frac{\lambda}{2}\bm{\theta}^{k+1}(\bm{P}^k)^T\bm{P}^k\bm{\theta}^{k+1}$ to both sides of the inequality, we get
\begin{align}\label{23}
\nonumber &\Psi(\bm{\theta}^{k+1}, \bm{P}^k) \\
\nonumber &= f(\bm{\theta}^{k+1}) + \frac{\lambda}{2}\bm{\theta}^{k+1}(\bm{P}^k)^T\bm{P}^k\bm{\theta}^{k+1}\\
\nonumber & \leq f(\bm{\theta}^{k}) + \langle \nabla f(\bm{\theta}^{k}), \bm{\theta}^{k+1} - \bm{\theta}^{k} \rangle + \frac{N L}{2} \|\bm{\theta}^{k+1} - \bm{\theta}^{k}\|^2 + \frac{\lambda}{2}\bm{\theta}^{k+1}(\bm{P}^k)^T\bm{P}^k\bm{\theta}^{k+1} \\
\nonumber & = f(\bm{\theta}^{k}) + \langle \nabla f(\bm{\theta}^{k}), \bm{\theta}^{k+1} - \bm{\theta}^{k} \rangle + \frac{N L}{2} \|\bm{\theta}^{k+1} - \bm{\theta}^{k}\|^2 + \frac{\lambda}{2}\bm{\theta}^{k+1}(\bm{P}^k)^T\bm{P}^k\bm{\theta}^{k+1}  \\
\nonumber & \quad \quad - \frac{\lambda}{2}\bm{\theta}^{k}(\bm{P}^k)^T\bm{P}^k\bm{\theta}^{k} + \frac{\lambda}{2}\bm{\theta}^{k}(\bm{P}^k)^T\bm{P}^k\bm{\theta}^{k} \\
\nonumber & = f(\bm{\theta}^{k}) + \langle \nabla f(\bm{\theta}^{k}), \bm{\theta}^{k+1} - \bm{\theta}^{k} \rangle + \frac{N L}{2} \|\bm{\theta}^{k+1} - \bm{\theta}^{k}\|^2 + \frac{\lambda}{2}(\bm{\theta}^{k+1} - \bm{\theta}^{k})^T(\bm{P}^k)^T\bm{P}^k\bm{\theta}^{k+1} \\
\nonumber & \quad \quad + \frac{\lambda}{2}\bm{\theta}^{k}(\bm{P}^k)^T\bm{P}^k\bm{\theta}^{k} \\
& = \Psi(\bm{\theta}^{k}, \bm{P}^k) + \langle \nabla_{\bm{\theta}} \Psi(\bm{\theta}^{k}, \bm{P}^k), \bm{\theta}^{k+1} - \bm{\theta}^{k} \rangle + \frac{N L}{2} \|\bm{\theta}^{k+1} - \bm{\theta}^{k}\|^2,
\end{align}
where the last equality follows from the definition of $\Psi(\bm{\theta}, \bm{P})$.

Using the update rule for $\bm{\theta}^{k+1}$, we have
\begin{align}\label{24}
\nonumber &\langle \nabla_{\bm{\theta}} \Psi(\bm{\theta}^k, \bm{P}^k), \bm{\theta}^{k+1} - \bm{\theta}^k \rangle \\
\nonumber &= \langle \nabla f(\bm{\theta}^k) + \lambda (\bm{P}^k)^T \bm{P}^k \bm{\theta}^k, -\alpha (\nabla f(\bm{\theta}^k) + \lambda (\bm{P}^k)^T \bm{P}^k \bm{\theta}^k) \rangle \\
\nonumber &= -\alpha \|\nabla f(\bm{\theta}^k) + \lambda (\bm{P}^k)^T \bm{P}^k \bm{\theta}^k\|^2 \\
&= -\alpha \|\nabla_{\theta} \Psi(\bm{\theta}^k, \bm{P}^k)\|^2.
\end{align}
On the other hand, we have
\begin{align}\label{25}
\nonumber &\|\bm{\theta}^{k+1} - \bm{\theta}^{k}\|_2^2 \\
\nonumber &= \alpha^2 \|\nabla f(\bm{\theta}^k) + \lambda (\bm{P}^k)^T \bm{P}^k \bm{\theta}^k\|^2 \\
&= \alpha^2 \|\nabla_{\theta} \Psi(\bm{\theta}^k, \bm{P}^k)\|^2.
\end{align}
Substituting (\ref{24}) and (\ref{25}) back into (\ref{23}), we obtain
\begin{align}\label{first}
\Psi(\bm{\theta}^{k+1}, \bm{P}^k) &\leq \Psi(\bm{\theta}^{k}, \bm{P}^k) - \alpha \left(1 - \frac{\alpha N L}{2}\right) \|\nabla_{\theta} \Psi(\bm{\theta}^k, \bm{P}^k)\|^2.
\end{align}
By choosing $\alpha < \frac{2}{N L}$, we ensure that the term $1-\frac{\alpha N L}{2}$ is positive.

From the definition of $\Psi(\bm{\theta}, \bm{P})$, we have
\begin{align}
\Psi(\bm{\theta}^{k+1}, \bm{P}^{k+1}) = f(\bm{\theta}^{k+1}) + \frac{\lambda}{2} (\bm{\theta}^{k+1})^T (\bm{P}^{k+1})^T \bm{P}^{k+1} \bm{\theta}^{k+1}.
\end{align}
Using the update rule for $\bm{P}^{k+1}$, i.e., $\bm{P}^{k+1} = \bm{H} \odot (\bm{P}^k - \eta \lambda \bm{P}^k \bm{\theta}^{k+1} (\bm{\theta}^{k+1})^T)$, we can write
\begin{align}
\nonumber&(\bm{P}^{k+1})^T \bm{P}^{k+1} \\
\nonumber& = (\bm{H} \odot (\bm{P}^k - \eta \lambda \bm{P}^k \bm{\theta}^{k+1} (\bm{\theta}^{k+1})^T))^T (\bm{H} \odot (\bm{P}^k - \eta \lambda \bm{P}^k \bm{\theta}^{k+1} (\bm{\theta}^{k+1})^T)) \\
&\preceq (\bm{P}^k - \eta \lambda \bm{P}^k \bm{\theta}^{k+1} (\bm{\theta}^{k+1})^T)^T (\bm{P}^k - \eta \lambda \bm{P}^k \bm{\theta}^{k+1} (\bm{\theta}^{k+1})^T),
\end{align}
where the inequality follows from the fact that the Hadamard product with $\bm{H}$ zeros out some entries, which can only decrease the Frobenius norm.

Expanding the right-hand side, we get
\begin{align}
\nonumber &(\bm{P}^k - \eta \lambda \bm{P}^k \bm{\theta}^{k+1} (\bm{\theta}^{k+1})^T)^T (\bm{P}^k - \eta \lambda \bm{P}^k \bm{\theta}^{k+1} (\bm{\theta}^{k+1})^T) \\
&= (\bm{P}^k)^T \bm{P}^k - 2\eta \lambda (\bm{P}^k)^T \bm{P}^k \bm{\theta}^{k+1} (\bm{\theta}^{k+1})^T + \eta^2 \lambda^2 (\bm{P}^k)^T \bm{P}^k \bm{\theta}^{k+1} (\bm{\theta}^{k+1})^T \bm{\theta}^{k+1} (\bm{\theta}^{k+1})^T
\end{align}

Substituting this back into the expression for $\Psi(\bm{\theta}^{k+1}, \bm{P}^{k+1})$, we obtain
\begin{align}
\nonumber &\Psi(\bm{\theta}^{k+1}, \bm{P}^{k+1}) \\
\nonumber &\leq f(\bm{\theta}^{k+1}) + \frac{\lambda}{2} (\bm{\theta}^{k+1})^T (\bm{P}^k)^T \bm{P}^k \bm{\theta}^{k+1} - \eta \lambda^2 (\bm{\theta}^{k+1})^T (\bm{P}^k)^T \bm{P}^k \bm{\theta}^{k+1} (\bm{\theta}^{k+1})^T \bm{\theta}^{k+1} \\
&\quad + \frac{\eta^2 \lambda^3}{2} (\bm{\theta}^{k+1})^T (\bm{P}^k)^T \bm{P}^k \bm{\theta}^{k+1} (\bm{\theta}^{k+1})^T (\bm{\theta}^{k+1}) (\bm{\theta}^{k+1})^T \bm{\theta}^{k+1}.
\end{align}
Since the gradient of $\Psi$ with respect to $\bm{P}$ is given by $\nabla_P \Psi(\bm{\theta}, \bm{P}) = \lambda \bm{P} \bm{\theta} \bm{\theta}^T$, we can compute the following norm
\begin{align}
\nonumber &\|\nabla_P \Psi(\bm{\theta}^{k+1}, \bm{P}^{k})\|_F^2 \\
\nonumber &= \Tr((\lambda \bm{P}^k \bm{\theta}^{k+1} (\bm{\theta}^{k+1})^T)^T (\lambda \bm{P}^k \bm{\theta}^{k+1} (\bm{\theta}^{k+1})^T)) \\
\nonumber &= \lambda^2 \Tr(\bm{\theta}^{k+1}(\bm{\theta}^{k+1})^T (\bm{P}^k)^T \bm{P}^k \bm{\theta}^{k+1} (\bm{\theta}^{k+1})^T ) \\
&= \lambda^2 (\bm{\theta}^{k+1})^T (\bm{P}^k)^T \bm{P}^k \bm{\theta}^{k+1} (\bm{\theta}^{k+1})^T \bm{\theta}^{k+1},
\end{align}
where we have used the cyclic nature of the trace operator $\Tr(\cdot)$.

Therefore, we have the following bound
\begin{align}\label{second}
\nonumber &\Psi(\bm{\theta}^{k+1}, \bm{P}^{k+1})\\
\nonumber &\leq \Psi(\bm{\theta}^{k+1}, \bm{P}^k) - \eta  \left(1 - \frac{\eta \lambda}{2} \|\bm{\theta}^{k+1}\|^2\right) \|\nabla_P \Psi(\bm{\theta}^{k+1}, \bm{P}^{k})\|^2 \\
&\leq \Psi(\bm{\theta}^{k+1}, \bm{P}^k) - \eta  \left(1 - \frac{\eta \lambda D_{\theta}^2}{2}\right) \|\nabla_P \Psi(\bm{\theta}^{k+1}, \bm{P}^{k})\|^2,
\end{align}
where we have used Assumption 2.

Hence, to ensure $\left(1 - \frac{\eta \lambda D_{\theta}^2}{2}\right)$ is positive, we choose the value of $\eta$ to be $\eta < \frac{2}{\lambda D_\theta^2}$.
Combining the inequalities (\ref{first}) and (\ref{second}), we get
\begin{align}
\nonumber &\Psi(\bm{\theta}^{k+1}, \bm{P}^{k+1})\\
&\leq \Psi(\bm{\theta}^{k}, \bm{P}^k) - \alpha \left(1 - \frac{\alpha N L}{2}\right) \|\nabla_{\theta} \Psi(\bm{\theta}^k, \bm{P}^k)\|^2  - \eta  \left(1 - \frac{\eta \lambda D_{\theta}^2}{2}\right) \|\nabla_P \Psi(\bm{\theta}^{k+1}, \bm{P}^{k})\|^2.
\end{align}
Summing up these inequalities from $k=0$ to $K-1$, we obtain
\begin{align}
\nonumber &\Psi(\bm{\theta}^{K}, \bm{P}^{K})\\
&\leq \Psi(\bm{\theta}^{0}, \bm{P}^0) - \alpha \left(1 - \frac{\alpha N L}{2}\right) \sum_{k=0}^{K-1} \|\nabla_{\theta} \Psi(\bm{\theta}^k, \bm{P}^k)\|^2 - \eta  \left(1 - \frac{\eta \lambda D_{\theta}^2}{2} \right) \sum_{k=0}^{K-1}\|\nabla_P \Psi(\bm{\theta}^{k+1}, \bm{P}^{k})\|^2.
\end{align}
Let $\Psi^\star$ be the optimal value of $\Psi$, and we define $\rho = \min\left\{\alpha \left(1 - \frac{\alpha N L}{2}\right), \eta  \left(1 - \frac{\eta \lambda D_{\theta}^2}{2} \right) \right\}$. Then, rearranging the terms, we can write
\begin{align}
\frac{1}{K}\sum_{k=0}^{K-1} \|\nabla \Psi(\bm{\theta}^k, \bm{P}^k)\|^2 &\leq \frac{1}{\rho K} (\Psi(\bm{\theta}^{0}, \bm{P}^0) - \Psi(\bm{\theta}^{K}, \bm{P}^{K})) \leq \frac{1}{\rho K} (\Psi(\bm{\theta}^{0}, \bm{P}^0) - \Psi^\star),
\end{align}
where we have used that $\Psi^\star \leq \Psi(\bm{\theta}^{K}, \bm{P}^{K})$ and  $\|\nabla \Psi(\bm{\theta}^k, \bm{P}^k)\|^2 = \|\nabla_{\bm{\theta}} \Psi(\bm{\theta}^k, \bm{P}^k)\|^2 + \|\nabla_{\bm{P}} \Psi(\bm{\theta}^k, \bm{P}^k)\|^2$. 
\end{proof}
\section{Additional Details on Datasets} \label{appendix:dataset}
 A summary of the datasets and the tasks used in Section \ref{experiment1} is presented in Table \ref{table:datasets}. These datasets are real-world datasets created in federated environments with varying degrees of heterogeneity. A detailed description of the datasets, along with their specific data partitioning schemes, is provided in Table \ref{table:nonIID}. To further quantify the non-IIDness in our data partitions, we have incorporated quantitative metrics assessing the degree of non-IIDness across different datasets in Table \ref{table:nonIID2}.
\begin{itemize}
    \item \textbf{Rotated MNIST (R-MNIST).} Following similar techniques as outlined in \citep{liu2022privacy}, we shuffle and then evenly separate the original MNIST dataset between 40 clients. Next, we randomly divide the clients into four groups, each containing 10 clients. We then apply rotations of \{0°, 90°, 180°, 270°\} to each group respectively. Therefore, clients within the same group share identical image rotations, resulting in the formation of four distinct clusters. The MNIST dataset is available under the CC BY-SA 3.0 license.
    \item \textbf{Heterogeneous CIFAR-10 (H-CIFAR-10).} The original CIFAR-10 dataset is split among 30 clients, and heterogeneity is introduced by assigning each client a random number of samples from 5 randomly selected classes out of the 10 available classes, following a similar approach as in \citep{t2020personalized, liu2022privacy}. The CIFAR-10 dataset is available under the MIT license.
    \item \textbf{Human Activity Recognition.} The dataset is composed of data gathered from the accelerometers and gyroscopes of smartphones used by 30 individuals, each performing one of six activities: walking, walking upstairs, walking downstairs, sitting, standing, or lying down. In this dataset, the data from each individual/client is treated as a unique task, with the primary objective being to differentiate between these activities. To identify each activity appropriately, feature vectors with 561 elements representing various time and frequency domain variables are used in the analysis. The dataset \citep{anguita2013public} is licensed under a Creative Commons Attribution 4.0 International (CC-BY 4.0) license.
    \item \textbf{Vehicle Sensor.} The dataset involves collecting data from a network of 23 wireless sensors, including acoustic (microphones), seismic (geophones), and infrared (polarized IR sensors), strategically placed along a specific road segment. This dataset aims to facilitate binary classification to identify two types of vehicles: the Assault Amphibian Vehicle (AAV) and the Dragon Wagon (DW). Each sensor, treated as a unique task or client, gathers acoustic and seismic data encapsulated in a 100-dimensional feature vector, representing the recordings as vehicles pass by. The Vehicle dataset was originally made public by its authors as a research dataset \citep{duarte2004vehicle}.
    \item \textbf{Google Glass Eating and Motion (GLEAM).} The dataset is collected using Google Glass from 38 individuals. It captures high-resolution sensor data to identify specific activities such as eating. This extensive dataset, consisting of 27,800 entries, each with a 180-dimensional feature vector, records head movements for binary classification to determine if the wearer is eating or not. The data includes accelerometer, gyroscope, and magnetometer readings, analyzed for statistical, spectral, and temporal characteristics to distinguish eating from other activities like walking, talking, and drinking. The GLEAM dataset, released by its original authors \citep{rahman2015unintrusive}, is available for non-commercial use.
    \item \textbf{School.} The dataset, originally introduced in \citep{goldstein1991multilevel}, seeks to forecast the exam results of 15,362 students from 139 secondary schools. The dataset contains information for each school, with student numbers ranging from 22 to 251, and each student is described using a 28-dimensional feature vector. This vector contains information about the school's ranking, the student's birth year, and the availability of free meals at the school. The dataset has been made publicly available in \citep{zhou2011malsar}.
\end{itemize}

\begin{table}[H]
\centering
\caption{
  Summary of the datasets and tasks used in our empirical setup.
  }
\label{table:datasets} 
\vspace{0.5em}
\begin{tabular}{@{}lcccc@{}}
\toprule
\textbf{Dataset} & \multicolumn{1}{l}{\textbf{Task}} & \multicolumn{1}{l}{\textbf{\# Clients/Tasks}} & \multicolumn{1}{l}{\textbf{Input Dimension}} & \textbf{Model} \\ \midrule
R-MNIST & Classification & 40 & $28 \times 28 \times 1$ & convolutional neural network\\
H-CIFAR-10 & Classification & 30 &  $32 \times 32 \times 3$ & convolutional neural network \\
HAR & Classification & 30 &561 & multinomial logistic regression\\
Vehicle Sensor & Classification & 23  & 100 & multinomial logistic regression\\
GLEAM & Classification & 38  & 180 & multinomial logistic regression \\
School & Regression & 139  & 28 & linear regression\\ \bottomrule
\end{tabular}
\end{table}

\begin{table}[t]
\centering
\caption{
Data partitioning strategies across datasets where $n_k$ is the number of
training samples of client $k$.
  }
\label{table:nonIID} 
\vspace{0.5em}

\begin{tabular}{@{}lccc@{}}
\toprule
\begin{tabular}[c]{@{}c@{}} \textbf{Dataset}\\ \textbf{ } \end{tabular} & \begin{tabular}[c]{@{}c@{}} \textbf{Data split}\\ \textbf{ }\end{tabular} & \begin{tabular}[c]{@{}c@{}} \textbf{Domain }\\ \textbf{distribution}\end{tabular} & \begin{tabular}[c]{@{}c@{}} \textbf{Label}\\ \textbf{distribution}\end{tabular}  \\ \midrule
\begin{tabular}[c]{@{}c@{}} R-MNIST\\ \textbf{ } \end{tabular}  & \begin{tabular}[c]{@{}c@{}} min $n_k = 1500$ \\ max $n_k = 1500$ \end{tabular}  & \begin{tabular}[c]{@{}c@{}} 4 groups with distinct\\  rotation angles \end{tabular} & \begin{tabular}[c]{@{}c@{}} Uniform distribution\\ within each rotation group \end{tabular} \\ \midrule 
\begin{tabular}[c]{@{}c@{}} H-CIFAR-10\\ \textbf{ } \end{tabular}  & \begin{tabular}[c]{@{}c@{}} min $n_k = 1515$ \\ max $n_k = 1839$ \end{tabular}  & \begin{tabular}[c]{@{}c@{}} Not explicitly \\ divided into domains \end{tabular} & \begin{tabular}[c]{@{}c@{}} Each client has data\\ from 5 random classes \end{tabular} \\ \midrule 
\begin{tabular}[c]{@{}c@{}} HAR\\ \textbf{ } \end{tabular}  & \begin{tabular}[c]{@{}c@{}} min $n_k = 210$ \\ max $n_k = 306$ \end{tabular}  & \begin{tabular}[c]{@{}c@{}} All activity \\ classes per client \end{tabular} & \begin{tabular}[c]{@{}c@{}} Uniform across activity\\ classes within each client \end{tabular}  \\ \midrule 
\begin{tabular}[c]{@{}c@{}} Vehicle Sensor\\ \textbf{ } \end{tabular}  & \begin{tabular}[c]{@{}c@{}} min $n_k = 872$ \\ max $n_k = 1933$\end{tabular}  & \begin{tabular}[c]{@{}c@{}} Same label set with\\ different feature distributions \end{tabular} & \begin{tabular}[c]{@{}c@{}} Balanced across vehicle\\ classes per sensor \end{tabular} \\ \midrule 
\begin{tabular}[c]{@{}c@{}} GLEAM\\ \textbf{ } \end{tabular}  & \begin{tabular}[c]{@{}c@{}} min $n_k = 699$ \\ max $n_k = 776$ \end{tabular}  & \begin{tabular}[c]{@{}c@{}} All activity classes\\ with uniform distribution \end{tabular} & \begin{tabular}[c]{@{}c@{}} Balanced between eating\\ and non-eating classes \end{tabular}  \\ \midrule 
\begin{tabular}[c]{@{}c@{}} School\\ \textbf{ } \end{tabular}  & \begin{tabular}[c]{@{}c@{}} min $n_k = 15$ \\ max $n_k = 175$ \end{tabular}  & \begin{tabular}[c]{@{}c@{}} Shared regression task\\ with uniform feature sets \end{tabular} & \begin{tabular}[c]{@{}c@{}} Continuous targets with\\ varying distributions per school \end{tabular} \\ \bottomrule
\end{tabular}
\end{table}

\begin{table}[t]
\centering
\caption{
  Degree of non-IIDness across datasets.
  }
\label{table:nonIID2} 
\vspace{0.5em}
\begin{tabular}{@{}lccc@{}}
\toprule
\textbf{Dataset} &  \textbf{Non-IID Metric} & \textbf{Description} & \textbf{Heterogeneity Type}   \\ \midrule
\begin{tabular}[c]{@{}c@{}} R-MNIST\\ \textbf{ } \end{tabular}  & \begin{tabular}[c]{@{}c@{}} Rotation angle \\ variance \end{tabular}  & \begin{tabular}[c]{@{}c@{}} High domain heterogeneity \\with 4 distinct rotation groups\end{tabular} & \begin{tabular}[c]{@{}c@{}} Feature distribution \\ skew \end{tabular} \\ \midrule 
\begin{tabular}[c]{@{}c@{}} H-CIFAR-10\\ \textbf{ } \end{tabular}  & \begin{tabular}[c]{@{}c@{}} Label distribution \\ \textbf{ }  \end{tabular}  & \begin{tabular}[c]{@{}c@{}} High label distribution  \\ skew among clients \end{tabular} & \begin{tabular}[c]{@{}c@{}} Label distribution \\ skew \end{tabular} \\ \midrule 
\begin{tabular}[c]{@{}c@{}} HAR\\ \textbf{ } \end{tabular}  & \begin{tabular}[c]{@{}c@{}} Inter-client \\ variability  \end{tabular}  & \begin{tabular}[c]{@{}c@{}} High heterogeneity due to unique \\  individual data per client \end{tabular} &  \begin{tabular}[c]{@{}c@{}} Concept shift\\  \end{tabular} \\ \midrule 
\begin{tabular}[c]{@{}c@{}} Vehicle Sensor\\ \textbf{ } \end{tabular}  & \begin{tabular}[c]{@{}c@{}} Feature distribution \\ \textbf{ } \end{tabular}  & \begin{tabular}[c]{@{}c@{}} High feature heterogeneity\\ across different sensors \end{tabular} & \begin{tabular}[c]{@{}c@{}} Feature distribution \\ skew \end{tabular} \\ \midrule 
\begin{tabular}[c]{@{}c@{}} GLEAM\\ \textbf{ } \end{tabular}  & \begin{tabular}[c]{@{}c@{}} Label distribution \\ \textbf{ } \end{tabular}  & \begin{tabular}[c]{@{}c@{}} Low heterogeneity \\ with balanced labels \end{tabular}  & \begin{tabular}[c]{@{}c@{}} Minimal \\ (almost IID) \end{tabular} \\ \midrule 
\begin{tabular}[c]{@{}c@{}} School\\ \textbf{ } \end{tabular}  & \begin{tabular}[c]{@{}c@{}} Continuous targets \\ variance \end{tabular}  & \begin{tabular}[c]{@{}c@{}} Moderate to high heterogeneity based \\ on inter-school performance variability \end{tabular} & \begin{tabular}[c]{@{}c@{}} Concept shift and \\ quantity skew \end{tabular}\\ \bottomrule
\end{tabular}
\end{table}

\section{Supplementary Experimental Results}\label{appendix:moreResults}
Figure \ref{vehicle_gleam_datasets} compares the performance of our proposed {\ours} method with dFedU on four datasets: (a) the Vehicle dataset, (b) the School dataset, (c) the HAR dataset and (d) the GLEAM dataset using $\gamma = \{0.1, 0.3\}$. The experiments evaluate the performance of the model in terms of communication rounds and transmitted bits. {\ours} and dFedU consistently achieves the highest accuracy and fastest convergence, demonstrating an average improvement of 5-10 percentage points over D-PSGD and local across datasets. For instance, on the Vehicle dataset, {\ours} reaches 88\% accuracy compared to 82\% for D-PSGD  and 82\% for local. While {\ours} requires more communication rounds for the Vehicle dataset to achieve similar test accuracy compared to dFedU, the performance gap between the two methods narrows as the number of communication rounds increases. However, when examining the number of transmitted bits, {\ours} demonstrates a clear advantage, requiring fewer bits to reach higher accuracy levels. For example, {\ours} with $\gamma = 0.1$ reaches a test accuracy of 0.85 with just 72 transmitted Kbits, while dFedU requires over 450 Kbits to approach this accuracy. Similar trends are observed for the other datasets. {\ours} recovers the same performance as dFedU in terms of test accuracy with a slight drop in test accuracy for the HAR and School datasets. Hence, the sheaf-based approach effectively captures the heterogeneous relationships among clients, achieving the same test accuracy using fewer transmitted bits than the baseline dFedU. 

\begin{figure}[t]
\centering
\begin{subfigure}[b]{\textwidth}
  \centering
  \includegraphics[scale=0.3]{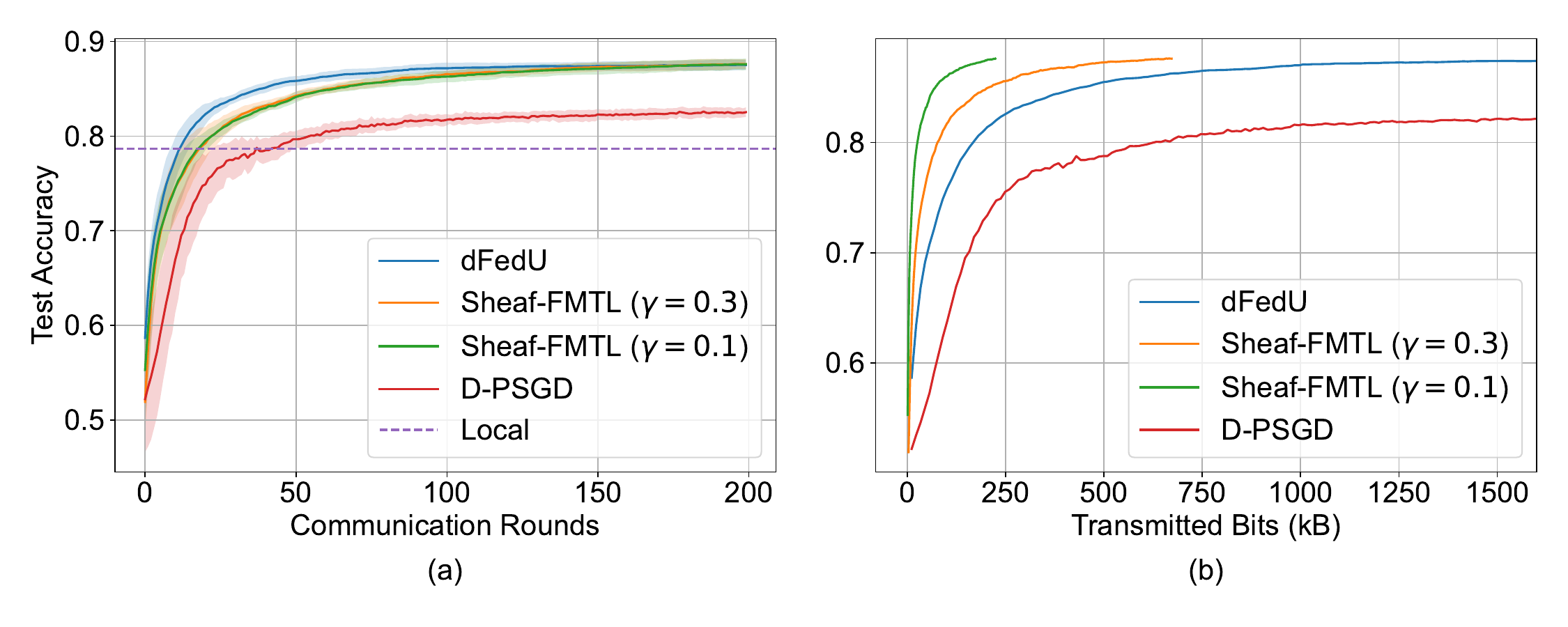}
\end{subfigure}%
\hfill
\begin{subfigure}[b]{\textwidth}
  \centering
  \includegraphics[scale=0.3]{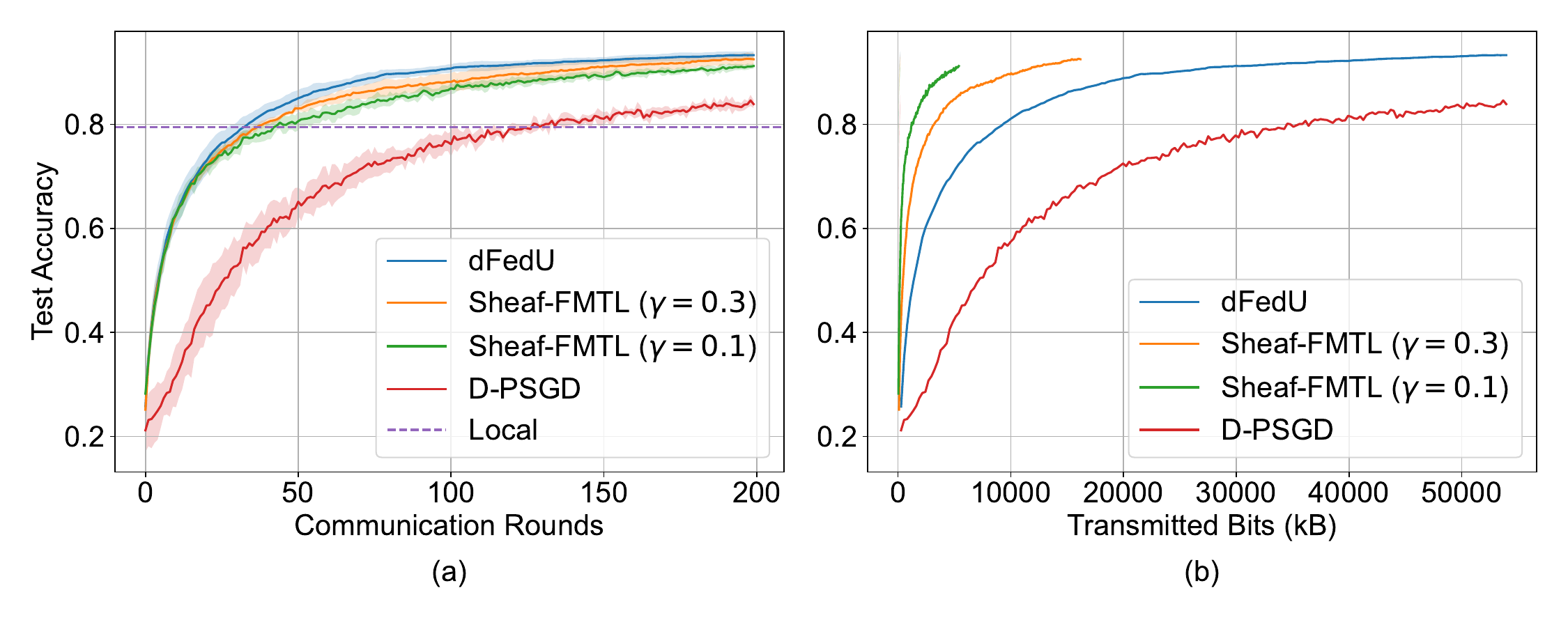}
\end{subfigure}%
\hfill
\begin{subfigure}[b]{\textwidth}
  \centering
  \includegraphics[scale=0.3]{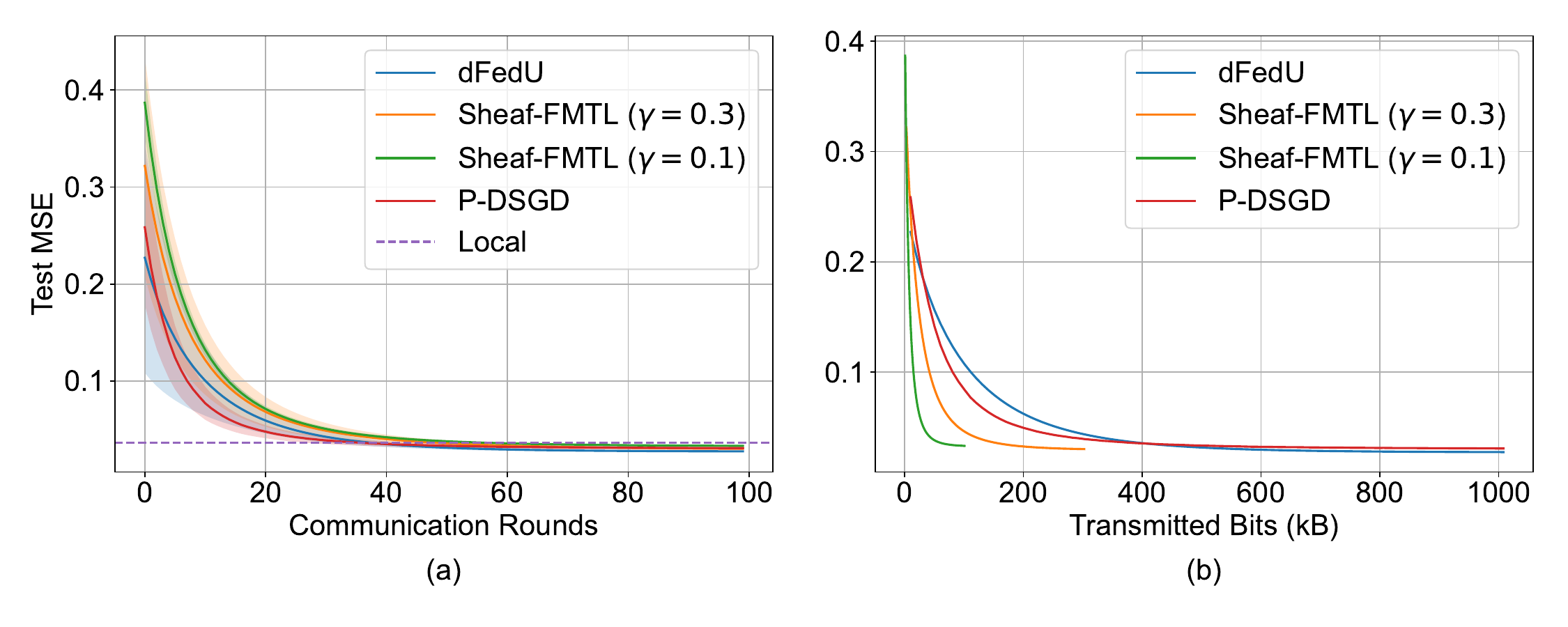}
\end{subfigure}%
\hfill
\begin{subfigure}[b]{\textwidth}
  \centering
  \includegraphics[scale=0.3]{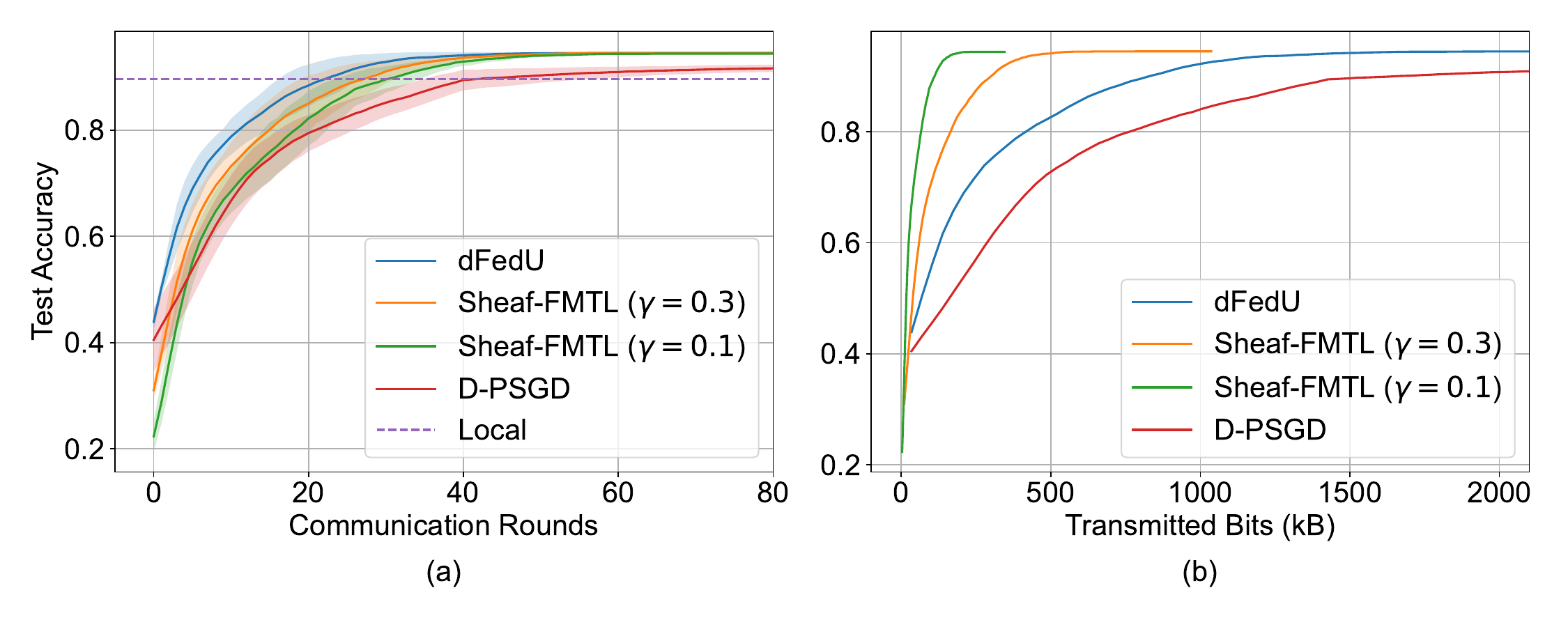}
\end{subfigure}%
\caption{Test/MSE accuracy as a function of the number of communication rounds and the number of transmitted bits for the Vehicle dataset (first row), the HAR dataset (second row), the School dataset (third row), and the GLEAM dataset (bottom).}
\label{vehicle_gleam_datasets}
\end{figure}

\end{document}